%% file: main.tex
\documentclass{article}
\PassOptionsToPackage{numbers,compress}{natbib}
\usepackage[preprint]{neurips_2026}

\usepackage[utf8]{inputenc}
\usepackage[T1]{fontenc}
\usepackage{amsmath,amssymb,amsfonts}
\usepackage{graphicx,booktabs,microtype}
\usepackage{listings}

\usepackage{xcolor}
\usepackage{xurl}
\usepackage{hyperref}
\hypersetup{
  pdftitle={TRIAGE: Evaluating Prospective Metacognitive Control in LLMs under Resource Constraints},
  pdfauthor={Zabir Al Nazi and Shubhashis Roy Dipta}
}

\title{\textsc{Triage}: Evaluating Prospective Metacognitive Control in LLMs under Resource Constraints}
\author{%
  Zabir Al Nazi \\
  University of California, Riverside, USA \\
  \texttt{znazi002@ucr.edu}
  \And
  Shubhashis Roy Dipta \\
  University of Maryland, Baltimore County, USA \\
  \texttt{sroydip1@umbc.edu}
}

\begin{document}
\maketitle

\begin{abstract}
\input{sections/abstract}
\end{abstract}

\section{Introduction}
\label{sec:introduction}
\input{sections/introduction}

\section{Related Work}
\label{sec:related_work}
\input{sections/related_work}

\section{TRIAGE Evaluation Framework}
\label{sec:dataset}
\input{sections/dataset}

\section{Experimental Setup}
\label{sec:experiments}
\input{sections/experiments}

\section{Results}
\label{sec:results}
\input{sections/results}

\section{Conclusion}
\label{sec:conclusion}
\input{sections/conclusion}

\begingroup
\interlinepenalty=10000 
\bibliographystyle{abbrvnat}
\bibliography{references}
\endgroup

\clearpage
\appendix
\section{Limitations and Broader Impact}
\label{sec:limitations}
\input{sections/limitations}

\clearpage
\section{Evaluation Documentation}
\label{app:documentation}
\input{sections/appendix_documentation}

\clearpage
\section{Additional Experiments}
\label{app:additional}
\input{sections/appendix_additional}

\end{document}

%% file: sections/abstract.tex
Deploying language models as autonomous agents requires more than per-task accuracy: when an agent faces a queue of problems under a finite token budget, it must decide which to attempt, in what order, and how much compute to commit to each, all before any execution feedback is available. This is the prospective form of metacognitive control studied for decades in human cognition, yet whether language models possess it remains largely untested. We introduce TRIAGE, an evaluation framework in which a model receives a task pool and a token budget calibrated to its own baseline cost, and commits to a single ordered plan that jointly encodes selection, sequencing, and per-problem allocation. Plans are scored against an oracle built from the model's measured solvability and token costs, using triage efficiency normalized between random and oracle performance, together with normalized regret relative to the oracle. We evaluate frontier and open-source models, with and without reasoning enabled, across competition mathematics, graduate-level science, code generation, and expert multidisciplinary knowledge, and find that current language models exhibit substantial gaps in prospective metacognitive control, revealing a previously unmeasured capability dimension with direct implications for resource-efficient agent deployment.

%% file: sections/introduction.tex
Scaling test-time compute has become a major paradigm for improving LLM reasoning~\citep{snell2024scaling}, but it has also made compute allocation a central unsolved problem: current reasoning models lack reliable \textit{prospective control} over their own compute expenditure, and adaptive allocation per problem can yield substantial efficiency gains~\citep{snell2024scaling,alomrani2025reasoning}. Reasoning models can expend excessive tokens even on trivial problems~\citep{chen2024overthinking} and underthink hard ones, prematurely abandoning promising reasoning paths~\citep{wang2025underthinking}. \textit{Unconstrained agents}, similarly, can consume thousands of tokens on agentic tasks and still fail, exhibiting analysis paralysis and wasted compute that correlates with decreased performance~\citep{cuadron2025overthinking}.

\hspace{1em}This creates a concrete deployment problem: agents built on these models face real resource constraints, including token budgets, tool-call limits, and latency windows, and must decide which tasks to attempt before committing resources. TRIAGE isolates the setting in which selection, ordering, and allocations are committed before execution feedback; interactive replanning is complementary. A resource-rational agent should allocate finite compute to the tasks where its expected return is highest~\citep{lieder2020resource}. In practice, this means an LLM facing a queue of problems under a token budget should abstain from attempting ones it cannot solve, spend less on easy ones, and safely commit sufficient resources to each attempted problem so that under-allocation does not expend budget on partial or incomplete solutions that yield no return. This is the same strategy a student applies on a timed exam, and an instance of the capability that \citet{nelson1990metamemory} formalized as \emph{metacognitive control}: regulating effort allocation based on judgments of one's own knowledge state.

\hspace{1em}Recent work suggests LLMs have some of this capability in isolation: large language models can predict whether they will answer a question correctly~\citep{kadavath2022language}, yet reasoning-trained models can be worse at abstaining on unanswerable questions~\citep{kirichenko2025abstentionbench}. \citet{barkan2025capable} evaluated prospective confidence and sequential task-acceptance decisions, finding systematic overconfidence that does not consistently diminish with scale or reasoning augmentation. The critical capability for an LLM planner is not just predicting token cost or knowing when to abstain from a single question, but managing a portfolio: given a set of tasks and a shared budget, selecting which to attempt, estimating cost, and ordering execution to maximize value. This is the problem that many real-world agentic deployments face: coding agents working through backlogs of issues, research agents decomposing queries into sub-questions of varying difficulty, and planning agents budgeting compute across subtasks of a larger goal all require cross-task prioritization under finite compute. Existing per-task budgeting methods estimate reasoning budgets for individual problems~\citep{han2025token,li2025selfbudgeter}, and external batch schedulers order requests for throughput~\citep{jin2023s3,fu2024scheduling}. However, whether LLMs can themselves perform the joint optimization of assessing their own feasibility, cost, and priority across a heterogeneous task set under a shared budget, before any execution feedback, remains largely untested.

\hspace{1em}We introduce TRIAGE, an evaluation framework for \textbf{prospective metacognitive control} under resource constraints. The construct has three components. \textbf{Metacognitive control} is the use of self-knowledge to regulate behavior, graded by the consequences of the model's plan rather than by verbalized confidence~\citep{ackerman2025limited}. \textbf{Prospective} denotes that all decisions are committed before any task is attempted, distinguishing committed resource-allocation decisions from retrospective measures of confidence or post-hoc calibration. Prospective confidence judgments are also studied in prior work~\citep{kadavath2022language,koriat1997monitoring}; TRIAGE requires a joint resource-allocation plan. \textbf{Resource constraints} refers to a finite shared budget, output tokens by default, that the model's behavior consumes~\citep{han2025token}. In TRIAGE, a model receives a set of problems, point values per problem, and a token budget calibrated to its own baseline cost on the set. It returns a single ordered plan specifying which problems to attempt, how many tokens to allocate to each, and in what order. We evaluate these plans using fixed, model-specific baseline outcomes and token costs under advisory and allocation-enforced scoring rules. This design operationalizes three foundational control functions from \citet{nelson1990metamemory} in their prospective form: selection (which problems to attempt), allocation (how much compute to commit per problem), and termination (where to commit budget cutoffs ex ante). TRIAGE evaluates the control output directly rather than asking the model to verbalize its underlying judgments, following work that tests LLM metacognition behaviorally rather than relying on self-reports~\citep{ackerman2025limited}. Just as human learners plan effort allocation under time pressure~\citep{son2000metacognitive}, an autonomous agent must jointly optimize these decisions to maximize expected utility under finite inference limits.

\hspace{1em}We evaluate 20 models in 30 configurations, including 10 models evaluated both with and without extended reasoning, across four domains: competition mathematics (AIME 2024--2025), graduate-level science (GPQA Diamond; \citealp{rein2024gpqa}), code generation (LiveCodeBench; \citealp{jain2024livecodebench}), and multidisciplinary expert knowledge (Humanity's Last Exam; \citealp{phan2025humanity}).

We find that current language models exhibit substantial and uneven gaps in prospective metacognitive control. Triage performance varies across domains, and extended reasoning has configuration-dependent effects on triage quality. Across repeated planning runs, performance falls substantially when per-problem allocations are enforced rather than treated as advisory. Allocation counterfactuals further diagnose budgeting errors, which persist in control pools containing only problems the model previously solved. The expanded GPQA-Abstain probe reveals substantial variation in unanswerability recognition, with paired standard-versus-extended comparisons favoring extended reasoning on average. Together, these experiments characterize a previously unmeasured capability dimension with direct implications for resource-efficient agent deployment.

%% file: sections/related_work.tex
\subsection{Metacognition in Large Language Models}

Metacognition, the capacity to monitor and regulate one's own cognitive processes~\citep{nelson1990metamemory}, has become a central concern in LLM evaluation. Research examines three aspects of metacognition: metacognitive \emph{knowledge} (can the model identify what a task requires?), metacognitive \emph{monitoring} (can the model judge whether it will succeed?), and metacognitive \emph{control} (can the model act on those judgments to regulate its behavior?). What remains largely untested is prospective control over a task portfolio: jointly selecting, ordering, and allocating compute under a shared budget before execution feedback.

\paragraph{Metacognitive knowledge and monitoring.}
\citet{kadavath2022language} provided the foundational result on monitoring, showing that LLMs can predict which questions they will answer correctly, with calibration improving at scale. \citet{didolkar2024metacognitive} demonstrated that LLMs also possess metacognitive \emph{knowledge}: given a math problem, frontier models can assign meaningful skill labels and use skill-matched exemplars to improve reasoning, establishing that models can assess task properties. \citet{xiong2024uncertainty} systematically benchmarked confidence elicitation methods and found that LLMs tend to be overconfident when verbalizing confidence, though both calibration and failure prediction improve with scale. \citet{wang2025dmc} formalized the separation of metacognitive ability from raw cognitive ability via their DMC framework, finding that stronger LLMs exhibit stronger metacognition and that higher measured metacognitive ability is associated with better abstention performance. At the mechanistic level, \citet{ji2025metacognitive} demonstrated that LLMs can monitor and control a low-dimensional subspace of their internal activations via a neurofeedback paradigm, providing direct evidence of metacognitive access to internal states, though limited to concurrent activations rather than prospective task-level judgments.

These results are counterbalanced by evidence of systematic failure. \citet{kapoor2024taught} showed that verbalized confidence is poorly calibrated for open-ended generation and that supervised fine-tuning through model features substantially improves uncertainty estimation. \citet{griot2025metacognition} found that LLMs provide confident answers to medical questions even when correct options are entirely absent, revealing a disconnect between stated confidence and actual capability.

\citet{ackerman2025limited} used behavioral paradigms where the model must decide before answering whether to attempt or defer (the Delegate Game) and found limited, context-dependent metacognition. \citet{steyvers2025metacognition} organize these mixed findings by distinguishing metacognitive \emph{sensitivity} (ranking own performance across items) from \emph{calibration} (matching confidence to success rate), noting that the two can vary independently and that neither is reliably strong across models and domains.

\paragraph{From monitoring to control.}
TRIAGE evaluates how task selection, ordering, and token allocation interact under a shared budget. The closest prior work to ours is \citet{barkan2025capable}, who explicitly distinguish in-advance from after-the-fact confidence and test prospective self-assessment on coding and agentic tasks (SWE-Bench Verified). They find systematic overconfidence, without consistent gains in prediction accuracy from scale or reasoning augmentation. Their resource-acquisition experiment tests sequential contract acceptance with success rewards, failure penalties, and feedback from prior outcomes. It does not require jointly selecting, ordering, and assigning token allocations across a visible task pool under one shared budget before feedback.

AbstentionBench~\citep{kirichenko2025abstentionbench} evaluates abstention across 20 datasets and reports an average ${\sim}$24\% reduction in abstention in its reasoning-fine-tuning comparisons. We use its GPQA-Abstain and MMLU-Math-Abstain variants in our unanswerable-item injection experiments. AbstentionBench evaluates abstention on individual questions; TRIAGE evaluates a joint commitment to task selection, ordering, and token allocation across a task pool under a shared budget.

Across this literature, a consistent picture emerges: LLMs possess metacognitive knowledge and some monitoring ability, but both are contested in degree and fragile across domains. TRIAGE evaluates the quality of a joint plan specifying task selection, execution order, and token allocations under a shared budget, committed before execution feedback. Per-task self-assessment alone is insufficient: the model must also compare feasibility and cost across tasks while respecting the shared constraint. Whether current models can make this transition from single-task monitoring to portfolio-level resource allocation is what TRIAGE is designed to test.

\subsection{Compute Allocation for LLM Reasoning}

The test-time compute scaling paradigm~\citep{snell2024scaling} has made compute allocation a practical problem. Reasoning models overthink trivial problems~\citep{chen2024overthinking}, underthink hard ones~\citep{wang2025underthinking}, and in agentic settings exhibit analysis paralysis that correlates with decreased performance~\citep{cuadron2025overthinking}. Existing methods address this waste, but split on who allocates and at what granularity, and neither side covers the full problem.

On the self-assessed side, TALE~\citep{han2025token} has the LLM estimate a token budget before solving a problem, and SelfBudgeter~\citep{li2025selfbudgeter} trains the model to emit a budget before its chain of thought. Both are prospective and self-assessed, but per-task: the model decides how long to spend, never whether to attempt at all, and there is no shared budget linking decisions across tasks. On the systems side, external predictors schedule LLM inference for serving efficiency through length-bucket classification~\citep{jin2023s3}, listwise ranking~\citep{fu2024scheduling}, or interval prediction under uncertainty~\citep{chen2025adaptive}. These operate at the portfolio level, but the predictor is external to the LLM, the objectives concern latency and throughput, and the model's own judgment of what it can solve plays no role.

Existing benchmarks test related but distinct capabilities. PlanBench~\citep{valmeekam2023planbench} evaluates plan generation, verification, and reasoning about execution in classical planning domains, with external validation; it does not evaluate a model's selection and token allocation across problems based on its own measured solvability and costs. $\tau$-bench~\citep{yao2024tau} evaluates tool-using agents in user interactions through policy-constrained tasks and final database-state checks, and GAIA~\citep{mialon2023gaia} provides a heterogeneous, difficulty-leveled agent benchmark with validated ground truth. TRIAGE instead evaluates an LLM's joint selection, ordering, and token allocation across a task pool under a shared budget, with the plan committed before execution feedback.

A gap remains in evaluating joint resource allocation under a shared budget. Single-task confidence benchmarks isolate monitoring; external schedulers isolate allocation; neither evaluates a model's joint selection, ordering, and token allocation under a shared budget. TRIAGE is designed to close this gap as an evaluation. The model produces a plan (selecting tasks, estimating costs, and ordering execution), which we score by simulating its consequences against fixed, model-specific baseline outcomes and token costs, using triage efficiency and normalized regret to assess the plan under a shared budget. Accurate per-task self-assessment alone is insufficient: the model must also compare feasibility and cost across tasks while respecting a shared constraint.

%% file: sections/dataset.tex
\providecommand{\sm}[1]{{\small $#1$}}

\label{sec:preliminaries}

\subsection{Prospective Metacognitive Control}
\label{sec:prospective-metacognitive-control}

Effective problem-solving under constraint requires more than the capacity to execute individual tasks; it requires metacognitive control over how that capacity is deployed, and for an agent operating under a finite execution budget, that control is exercised as prospective planning: a commitment, made before any problem is attempted, to which problems will be pursued, in what order, and how much of the budget each will receive. Cognitive psychology distinguishes an \emph{object level}, where a task is executed, from a \emph{meta level}, which monitors the object level and exerts control over it~\citep{nelson1990metamemory, koriat1996monitoring}. Two control functions of the meta level, what to attempt and how long to persist on what is attempted, are the levers behind human study-time allocation~\citep{son2000metacognitive, metcalfe2005region} and the metacognitive regulation of study effort before testing~\citep{ackerman2011metacognitive}. We benchmark functional analogues in language models.

Relative to task execution, TRIAGE distinguishes prospective decisions committed before execution from concurrent decisions made during execution and retrospective judgments made afterwards. We focus on the prospective setting, analogous to pre-study Ease-of-Learning judgments in the laboratory~\citep{nelson1990metamemory,leonesio1990different} and initial test plans in applied settings, for three independent reasons.

\begin{description}
\item[Methodological isolation.]
Once feedback is allowed during execution, observed performance can reflect both prospective judgment quality and adaptive replanning. Evaluating a committed plan without intervening feedback isolates the up-front decision from later revisions. Human studies likewise use controlled study-time conditions to distinguish metacognitive regulation from learning performance~\citep{ackerman2011metacognitive}.

\item[Deployment realism.]
LLM-driven systems already make task-level resource commitments before execution outcomes are known. Deep research agents can decompose queries into parallel sub-agents with separate context windows and query-dependent effort guidelines, with the lead agent waiting for a batch to finish before proceeding~\citep{hadfield2025multiagent}. Deep research agents also issue batches of tool calls before receiving their results; W\&D studies how the number and scheduling of parallel calls affect performance and efficiency~\citep{lin2026wide}. Coding agents such as SWE-agent operate on repository issues under preassigned per-instance inference-cost budgets, automatically submitting existing edits when the budget is exceeded~\citep{yang2024sweagent}. Under a shared deployment budget, resources spent on unsuccessful attempts reduce what remains for other work. TRIAGE isolates the corresponding LLM planning decisions across a task pool: what to attempt, in what order, and with what token allocations.

\item[Construct priority.]
Replanning revises earlier prospective decisions in light of execution feedback. To measure what mid-execution adjustments add, we first need a clean measurement of the up-front plan in isolation, and that requires evaluating planning without any feedback during execution.
\end{description}

\paragraph{Tokens as the control variable.}
Classical study-allocation work uses time as the meta-level control variable. In autoregressive language models, output tokens are the natural unit of expended effort: generation time, compute, and monetary cost are all approximately proportional to token count, and the central paradigm of test-time compute scaling treats per-problem token allocation as the system's primary lever for adapting to problem difficulty~\citep{snell2024scaling, han2025token}. Tokens are therefore the operational analog of time in study-allocation theory, with the additional property of being exactly measurable.

\paragraph{Distributional self-knowledge and binding plans.}
Agent pipelines can decouple planning from execution into separate inference calls~\citep{xu2023rewoo}. In TRIAGE, the planner does not introspect on an ongoing computation; it predicts the behavior of a solver using the same model checkpoint under standard execution conditions. We treat this as the deployment-relevant form of self-modeling. The plan is \emph{operationally binding} in regime E: declared allocations are enforced by the scoring rule against the fixed reference completion, converting the forecast into a resource commitment. We study this enforced setting alongside the unconstrained one as a secondary regime (\S\ref{sec:problem-formulation}).

\paragraph{Four primitives.}
Prospective metacognitive control under a finite execution budget reduces to four jointly exercised judgments:
(i) \emph{feasibility}, whether the model can solve the problem?
(ii) \emph{cost}, how many tokens will it take?
(iii) \emph{selection}, which subset to attempt?
and (iv) \emph{sequencing}, in what order, given that budget exhaustion truncates the tail?
Existing LLM benchmarks isolate slices of this set: per-item confidence calibration covers (i)~\citep{barkan2025capable}; abstention covers a hard form of (i) without a budget~\citep{kirichenko2025abstentionbench}; per-task budget control covers (ii) for a single item~\citep{li2025selfbudgeter}. TRIAGE measures the joint exercise of all four primitives under a finite budget.

\subsection{Problem Formulation}
\label{sec:problem-formulation}

We formalize TRIAGE for a single \emph{task pool}, a finite set of problems sharing a domain and a budget, the unit of work the agent commits to in a single planning step before any execution feedback arrives.

\paragraph{Setup.}
A task pool \sm{P = \{p_1, \ldots, p_n\}} contains finitely many problems. For each problem \sm{p_i} and a fixed solver configuration \sm{M}: \sm{v_i \in \mathbb{R}_{>0}} is the point value, observable to \sm{M}; \sm{y_i \in \{0,1\}} records whether \sm{M} solved \sm{p_i} in its Phase-0 reference completion, unobservable to \sm{M} at planning time; and \sm{c_i \in \mathbb{Z}_{>0}} is the recorded output-token cost of that completion, also unobservable at planning time. These reference measurements remain fixed during scoring. The solver configuration specifies the model checkpoint, reasoning mode, sampling settings, and generation limits (Appendix~\ref{app:models}). Repeated-run analyses are described in Appendix~\ref{app:replication}.

\paragraph{Budget.}
The budget is \sm{B(M,\alpha)=\lfloor\alpha\cdot\sum_{i=1}^{n}c_i\rfloor} with \sm{\alpha\in(0,1]}; the floor reflects that token counts are integer-valued. For \sm{\alpha<1}, the budget is below the solver's measured full-pool cost and creates selection pressure. At \sm{\alpha=1}, every complete ordering fits under regime U, so the random and oracle references coincide; normalized regret supports substantive full-budget comparisons. The budget applies to simulated task-solving output costs; input tokens and tokens used to generate the plan are not deducted from \sm{B}.

\paragraph{Action.}
The model produces an ordered plan \sm{\pi=(i_1,\ldots,i_k)} with \sm{0\le k\le n}, indexing a selected subset \sm{S=\{i_1,\ldots,i_k\}\subseteq\{1,\ldots,n\}}, together with per-problem allocations \sm{a_{i_j}\in\mathbb{Z}_{\ge0}}. The prompt requires \sm{\sum_j a_{i_j}\le B}. Parsed plans are scored under regimes U and E, including when their declared allocations exceed the requested total. The plan thus specifies task selection, execution order, and token allocations.

\paragraph{Execution.}
Plans are scored in plan order against the fixed reference outcomes and costs under two regimes, which differ in whether the per-problem allocations \sm{a_i} are binding.

\label{eta_u}
\textit{Unconstrained regime (U).} Each reached problem incurs its measured reference cost \sm{c_i}; only the global budget \sm{B} binds. Scoring proceeds through \sm{\pi} until reaching an item whose measured cost exceeds the remaining budget, at which point scoring stops. Problem \sm{i_j} contributes \sm{v_{i_j}y_{i_j}} to the achieved value if reached, and zero otherwise. Per-problem allocations are advisory.

\label{eta_e}
\textit{Constrained regime (E).} Each positive-allocation problem \sm{i_j} is scored against its fixed reference completion: it contributes \sm{v_{i_j}y_{i_j}} if \sm{a_{i_j}\ge c_{i_j}}, and zero otherwise. The full allocation \sm{a_{i_j}} is deducted from the remaining budget regardless of outcome. Scoring proceeds through \sm{\pi} until reaching an item whose allocation exceeds the remaining budget, at which point scoring stops.

The two regimes probe complementary aspects of prospective control~\citep{nelson1990metamemory}. In regime U, feasibility and cost judgments matter through selection and sequencing, while per-problem allocations are advisory. In regime E, allocations additionally bind under the scoring rule, making the planner's resource commitments consequential. We write \sm{V_M(\pi)} for achieved value under the specified regime and report results in both.

\label{triage}

\paragraph{Measuring triage skill.}

We use normalized measures to evaluate how well the planner's committed plan \sm{\pi} acts on its predictions of \sm{(y_i,c_i)} under budget \sm{B}. A raw value comparison \sm{V_M(\pi)} is uninformative on its own: its magnitude depends on pool difficulty and on \sm{B}, which vary across cells. Following the reference-based logic of forecast skill scores~\citep{murphy1974sample}, we normalize achieved value between an uninformed random reference and a fully informed oracle. We also report normalized regret relative to the oracle.

\paragraph{Oracle and random references.}

The \emph{oracle} is a planner with full knowledge of \sm{(y_i, c_i)} for every~\sm{i}. Its achievable value is
\begin{equation}
  V_{\text{oracle}} \;=\;
  \max_{x \in \{0, 1\}^{n}} \;\sum_{i=1}^{n} v_i\, y_i\, x_i
  \quad \text{s.t.} \quad
  \sum_{i=1}^{n} c_i\, x_i \le B,
  \label{eq:oracle}
\end{equation}
a 0--1 knapsack - each problem is either included (\sm{x_i = 1}) or omitted (\sm{x_i = 0}), with no fractional attempts~\citep{kellerer2004knapsack}. The random reference \sm{V_{\mathrm{random}}} is estimated as the mean achieved value of \sm{10^3} uniformly random permutations of the gradeable items, with a fixed seed (42). Each permutation follows regime U: items incur their measured costs, and scoring stops at the first item whose cost exceeds the remaining budget. We use this same U-based random-order reference for both regimes.

\paragraph{Triage efficiency ratio.}

The triage efficiency ratio normalizes the model's achieved value between these two references:
\begin{equation}
  \eta_M \;=\;
  \frac{V_M(\pi) - V_{\text{random}}}
       {V_{\text{oracle}} - V_{\text{random}}}.
  \label{eq:eta}
\end{equation}
When \sm{V_{\mathrm{oracle}}>V_{\mathrm{random}}}, \sm{\eta_M=0} matches the performance of uninformed random ordering, while \sm{\eta_M=1} matches the oracle value attainable with perfect knowledge of the solver's reference outcomes and costs. Intermediate values express the fraction of improvement from random to oracle performance. Negative values indicate worse-than-random triage, for example when resources are spent on unsuccessful attempts or when ordering and allocation choices prevent solvable items from earning value. Boundary cases and their reporting conventions are detailed in Appendix~\ref{app:eta-boundary}.

\paragraph{Normalized regret.}
We also report the fraction of oracle value not captured:
\begin{equation}
  \widetilde R_M =
  \frac{V_{\mathrm{oracle}}-V_M(\pi)}{V_{\mathrm{oracle}}},
  \qquad V_{\mathrm{oracle}}>0.
  \label{eq:regret}
\end{equation}
It ranges from 0 at oracle performance to 1 when no oracle value is captured. Normalized regret complements efficiency by measuring the remaining gap to the oracle independently of the random reference, and supports comparisons at the full reference budget. Filtering details are provided in Appendix~\ref{app:eta-boundary}.

\paragraph{Uniform values: preventing difficulty leakage.}

We fix \sm{v_i = 1}. Any \sm{v_i} correlated with empirical difficulty would let the planner read difficulty off the prompt rather than estimate it from problem content, and language models are known to exploit such surface cues when present~\citep{tang2023large}. Under uniform \sm{v}, the oracle objective simplifies: \sm{V_{\text{oracle}}} equals the maximum number of solvable items (\sm{y_i = 1}) whose costs fit within~\sm{B}. Let \sm{S^{\star} = \{i : y_i = 1\}} and let \sm{i_{(1)}, i_{(2)}, \ldots} index \sm{S^{\star}} in ascending order of \sm{c_i}; then
\begin{equation}
  V_{\text{oracle}} \;=\;
  \max\Bigl\{\,m \,\Bigm|\, \sum_{j=1}^{m} c_{i_{(j)}} \le B\,\Bigr\}.
  \label{eq:oracle-uniform}
\end{equation}

The model receives problem text, point values, and the budget, and must form its plan without observing the per-problem reference outcomes or costs. The oracle algorithm, random-reference procedure, and scoring rules are shared across configurations, while their numerical reference values are model-specific. TRIAGE evaluates prospective metacognitive control through the quality of these model-specific resource-allocation plans. The oracle supplies the attainable reference value given full knowledge of the solver's recorded outcomes and costs.

%% file: sections/experiments.tex
\paragraph{Models.}
We evaluate 20 models yielding 30 model--mode configurations across frontier and open-source families (Qwen~2.5/3/3.5, GPT-OSS, Claude, DeepSeek, Gemini, GLM, GPT-5, Kimi). Ten models are evaluated in both standard inference and extended-reasoning configurations. Full model identifiers, providers, size labels, reasoning configurations, and serving settings are provided in Appendix~\ref{app:documentation}.

\paragraph{Datasets and task pools.}
We construct task pools from four domains: competition mathematics (AIME 2024--2025), graduate-level science (GPQA Diamond~\citep{rein2024gpqa}), code generation (LiveCodeBench~\citep{jain2024livecodebench}), and multidisciplinary expert knowledge (Humanity's Last Exam~\citep{phan2025humanity}). LiveCodeBench uses 2025-dated problems from the evaluated release\_v6 snapshot. We retain text-only HLE items outside mathematics, biology/medicine, and chemistry/materials. Each dataset is partitioned into task pools of nominal size $n_{\text{chunk}}=30$ to limit planner context length, motivated by prior evidence of long-context degradation~\citep{liu2024lostmiddle}. The partitioning procedure is specified in Additional Experiments. Each pool has its own budget, oracle, and random reference. Scores are averaged equally over available scored pools within each model configuration, dataset, budget condition, and planning run; HLE combines pools across its retained categories.

To probe sensitivity to unanswerable items, we mix answerable items with annotated unanswerable items from AbstentionBench~\citep{kirichenko2025abstentionbench}. The main unanswerable-item probe evaluates four configurations on GPQA-Abstain and MMLU-Math-Abstain; the expanded GPQA-Abstain evaluation covers all 30 main configurations, including ten matched standard/extended pairs. These probes use fixed task pools at $\alpha=1$, with nonzero injection ratios $r\in\{0.25,0.5,0.75,1.0\}$.

\paragraph{Phase~0: baseline measurement.}
We first obtain one completion per model--problem pair in an isolated session, recording pass@1 correctness $y_i$ and output-token cost $c_i$. These measured costs calibrate the per-pool budget $B_{\text{local}}=\lfloor\alpha\sum_{i\in\text{pool}}c_i\rfloor$ given to the planner in Phase~1. The pair $(c_i,y_i)$ provides the fixed reference measurements used to score plans and compute the oracle, efficiency, and regret. Sampling settings, serving limits, and token-accounting details are documented in Appendix~\ref{app:documentation}.

\paragraph{Phase~1: planning.}
The planner receives problem texts, point values ($v_i=1$), the per-pool budget $B_{\text{local}}=\lfloor\alpha\sum_{i\in\text{pool}}c_i\rfloor$, and instructions to return an ordered plan with per-problem token allocations summing to at most $B_{\text{local}}$. The full prompt template is given in Appendix~\ref{app:documentation}. Appendix~\ref{app:additional} reports a sensitivity analysis using three prompt framings for two models in both inference modes.

\paragraph{Phase~2: execution.}
The same submitted plan is scored under regimes U and E using the fixed Phase-0 reference outcomes and costs (\S\ref{sec:problem-formulation}), allowing us to compare the consequences of advisory and binding allocations.

\paragraph{Budget conditions.}
We sweep $\alpha\in\{0.25,0.5,0.75,1.0\}$ to vary selection pressure from severe to slack. We focus on $\alpha=0.5$ to evaluate prioritization when the budget cannot cover every problem at its reference cost.

\paragraph{Repeated-run evaluation.}
At $\alpha\in\{0.5,1.0\}$, repeated-run analyses use up to three independently sampled plans per configuration, answering the same prompts and scored against the fixed reference outcomes, costs, and prompt budgets. Separate Phase-0 repetitions assess reference-label and cost stability on AIME, GPQA Diamond, LiveCodeBench, and the Humanities/Social Science portion of HLE. The repeated-run experiments request temperature~0.7, except for Haiku extended reasoning at~1.0 and GPT-5 Mini using the provider default. We report means and sample standard deviations across available planning runs; sampling settings and evaluation coverage, including provider routes and serving-limit differences, are documented in Additional Experiments.

\paragraph{Metric.}
\label{metric}
We report triage efficiency \sm{\eta_M} (\S\ref{sec:problem-formulation}) as the primary metric at $\alpha=0.5$ and normalized regret \sm{\tilde{R}_M=(V_{\text{oracle}}-V_M)/V_{\text{oracle}}\in[0,1]}, for \sm{V_{\text{oracle}}>0}, as a complementary metric and the basis for full-budget comparisons; lower \sm{\tilde{R}_M} is better. The random reference \sm{V_{\text{random}}} is estimated by drawing \sm{10^3} uniform shuffles of the gradeable items per pool and executing each under regime-U rules, using this same U-based random reference for both regimes. Boundary and filtering conventions are described in Appendix~\ref{app:eta-boundary}.

For the unanswerable-item probes, planned-token waste \sm{W=\sum_{i\in S_+\cap U}a_i\,/\,\sum_{i\in S_+}a_i} measures allocations assigned to injected unanswerable items \sm{U}, where \sm{S_+=\{i:a_i>0\}}. The pilot uses positive-allocation detection \sm{D_+=|U\setminus S_+|/|U|}. The expanded GPQA probe uses listed-ID detection \sm{D_{\mathrm{listed}}=|U\setminus S_{\mathrm{listed}}|/|U|}, where \sm{S_{\mathrm{listed}}} contains the problem IDs listed in the plan, including zero-token entries. Its answerable-item coverage is \sm{C_{\mathrm{listed}}=|S_{\mathrm{listed}}\setminus U|/|P\setminus U|} for pool \sm{P}, and selectivity-corrected discrimination is \sm{J_{\mathrm{listed}}=D_{\mathrm{listed}}+C_{\mathrm{listed}}-1}. Expanded-probe metrics average pools within each injection ratio and then ratios; coverage and discrimination use ratios containing answerable items.

%% file: sections/results.tex
\subsection{Triage under a Shared Budget}
\label{sec:results-moderate}

\begin{figure}[t]
\centering
\includegraphics[width=\linewidth]{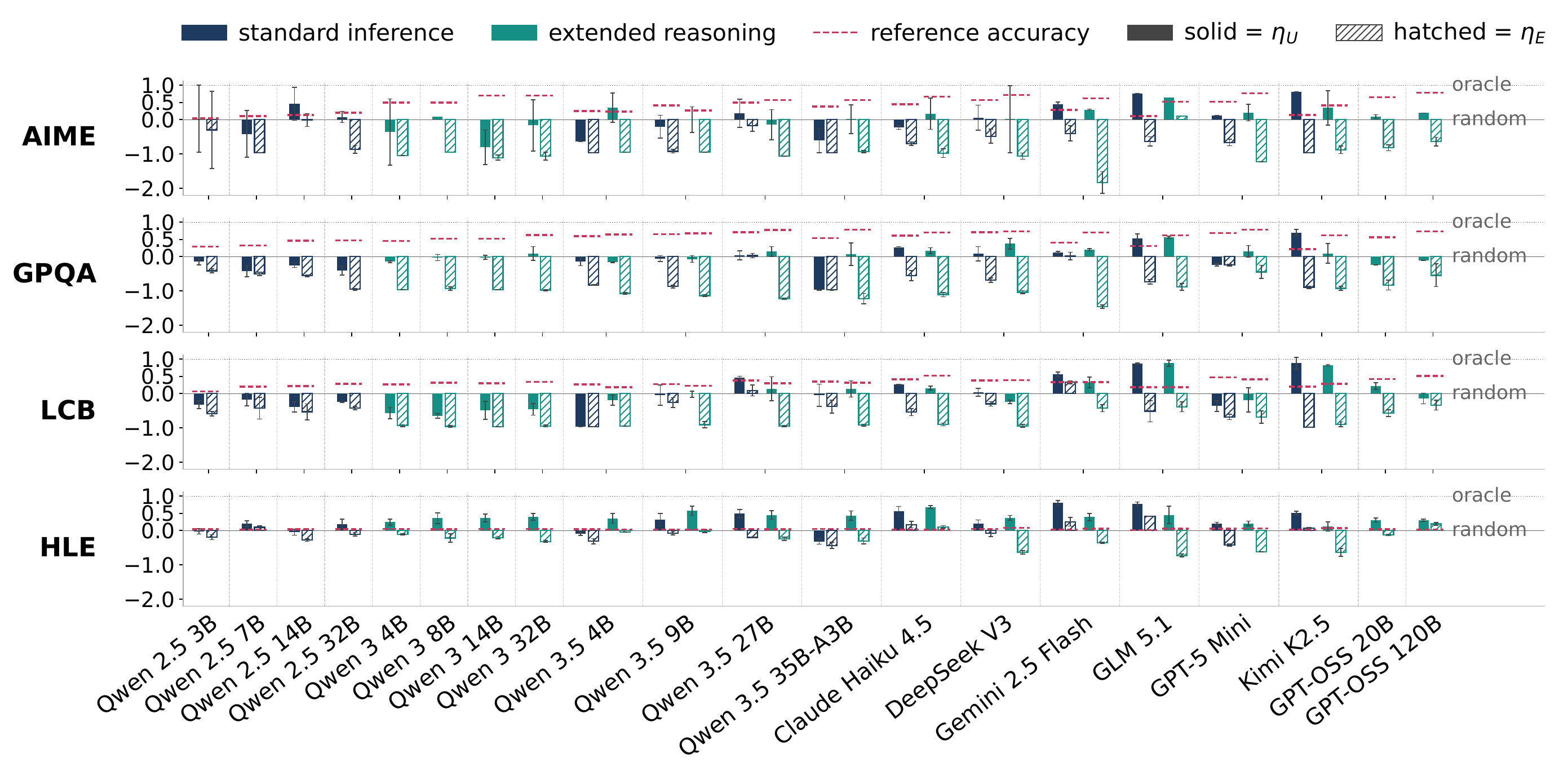}
\caption{\textbf{Triage efficiency across models and benchmarks under a shared budget ($\alpha=0.5$).} Solid bars show advisory execution ($\eta_U$), hatched bars show binding allocations ($\eta_E$), and colors distinguish standard inference from extended reasoning. Bars show means over up to three planning runs; whiskers give sample standard deviations. Dashed red lines show reference pass@1 accuracy from the single-run comparison. Higher efficiency is better. Planning-run coverage and scoring details are provided in \hyperref[app:additional]{Additional Experiments}.}
\label{fig:triage_efficiency}
\end{figure}

Figures~\ref{fig:triage_efficiency} and~\ref{fig:triage_regret} report triage across 30 model--mode configurations and four benchmarks at $\alpha=0.5$. We focus on this operating point to expose selection pressure: each pool receives half the sum of its measured reference token costs, so the budget cannot cover completing every item at its reference cost. The planner must decide which items to prioritize and how to distribute the available tokens. Figure~\ref{fig:triage_efficiency} reports efficiency on the random-to-oracle scale; Figure~\ref{fig:triage_regret} in Appendix~\ref{app:additional} reports the fraction of oracle value missed. The single-plan budget sweep and full-budget comparisons are provided in \hyperref[app:additional]{Additional Experiments}.

\paragraph{Advisory triage is uneven across benchmarks.}
Advisory efficiency in Figure~\ref{fig:triage_efficiency} varies sharply across domains, and no configuration leads all four benchmarks in any planning run. Performance can exceed the random reference on one benchmark and fall below it on another. Strong advisory performance in one domain therefore does not consistently carry over to the others. Figure~\ref{fig:triage_regret} shows the corresponding variation in achieved oracle value. Variability across planning runs is most pronounced on AIME, where each run-level estimate averages at most two scored pools.

\paragraph{Binding allocations sharply reduces triage performance.}
The separation between advisory and enforced execution is the most consistent pattern in Figures~\ref{fig:triage_efficiency} and~\ref{fig:triage_regret}. Across the 30 configurations and four benchmarks, mean efficiency is $\eta_U=0.10$ under advisory allocations and $\eta_E=-0.60$ when allocations bind. Advisory efficiency exceeds enforced efficiency in 118 of the 120 comparisons. Normalized regret confirms the separation: mean $\widetilde R_U=0.50$ versus $\widetilde R_E=0.88$. Thus, the same plans miss approximately one half of oracle value under advisory execution and nearly nine tenths when their declared allocations are enforced.

\paragraph{Extended reasoning lifts accuracy without consistently lifting triage.}
Within matched models, the reference-accuracy curves in Figures~\ref{fig:triage_efficiency} and~\ref{fig:triage_regret} generally rise with extended reasoning. The corresponding triage scores do not improve uniformly, and higher reference accuracy does not consistently produce better triage. Increased solving capability therefore does not guarantee better prospective control: extended reasoning can improve advisory planning while increasing its separation from enforced performance.

\paragraph{Better allocations improve triage within the same budget.}
With planned selection, order and total budget fixed, we rescore single-plan half-budget results. One rule assigns each selected problem its measured reference cost plus 20\%; another gives each an equal share of the budget. Each allocation, including the margin, is charged against the unchanged shared budget, and scoring stops when the next allocation cannot be funded. These rules recover 89\% and 57\% of the matched advisory--enforced efficiency gap, respectively. These gains without additional budget show that allocation choices contribute to the deficit. Details appear in \hyperref[app:additional]{Additional Experiments}.

\paragraph{Allocation failures persist on problems the model previously solved.}
To test whether triage failures persist without reference-unsolved items, we request new half-budget plans for pools containing only successful Phase-0 completions, using ten API-served configurations across AIME, GPQA, and LiveCodeBench. Enforced efficiency remains negative in 66 of the 71 pools with valid plans, and median enforced efficiency is $-1.69$. Allocation failures thus persist even when every item has a successful reference completion. Control construction and allocation diagnostics appear in \hyperref[app:additional]{Additional Experiments}.

\paragraph{Behavior across budget pressure.}
At $\alpha=1$, normalized regret is 0.29 under advisory allocations and 0.82 under enforcement, so substantial losses persist even at the full reference budget. Regret measures the fraction of reference-solvable value missed, remaining interpretable when the random and oracle references coincide. The full budget sweep appears in \hyperref[app:additional]{Additional Experiments}.

\subsection{Recognizing Unanswerable Problems}
\label{sec:unsolvable-injection}

\begin{figure}[t]
\centering
\setlength{\abovecaptionskip}{4pt}
\includegraphics[width=\linewidth]{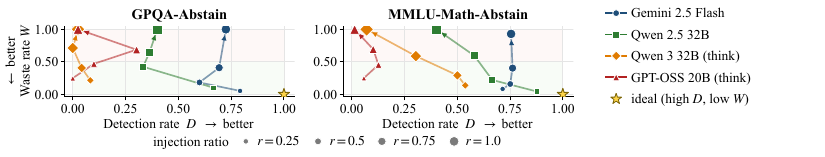}
\caption{\textbf{Trajectories through $(D,W)$ space as the unsolvable-injection ratio $r$ increases from $0.25$ to $1.00$ (marker size).} The star marks the ideal corner: high detection rate, low waste rate.}
\label{fig:unsolvable_trajectory}
\end{figure}

Effective planning also requires recognizing questions that cannot be answered from the supplied information. Allocating tokens to such questions commits resources that could otherwise support answerable work. We evaluate this ability by injecting annotated unanswerable items from AbstentionBench~\citep{kirichenko2025abstentionbench} (GPQA-Abstain and MMLU-Math-Abstain) into the corresponding task pools while holding pool size fixed. Figure~\ref{fig:unsolvable_trajectory} traces four configurations at injection ratios $r\in\{0.25,0.5,0.75,1.0\}$, using $\alpha=1$. Its detection axis $D$ denotes positive-allocation detection $D_+$: the fraction of injected items receiving no positive allocation. Planned-token waste $W$ is the fraction of declared token allocations assigned to those items.

\paragraph{Recognition varies across models.}
Figure~\ref{fig:unsolvable_trajectory} shows substantial differences among the four configurations. Gemini~2.5 Flash maintains relatively high detection across both datasets, while GPT-OSS~20B's detection remains low. Qwen~2.5 32B has higher detection than Qwen~3 32B at every injection ratio on both datasets.

\paragraph{Detection depends on injection ratio and domain.}
Detection does not decline uniformly as injection increases. Qwen~2.5 32B's detection declines monotonically on MMLU-Math-Abstain, but falls and then partially recovers on GPQA-Abstain. Qwen~3 32B also shows declining detection as injection increases on MMLU-Math-Abstain. Both Qwen configurations achieve higher average detection on MMLU-Math-Abstain than on GPQA-Abstain, demonstrating that recognition patterns also vary across domains.

\paragraph{Models allocate budget even when every item is unanswerable.}
At full injection, every item is unanswerable, so every positive allocation goes to an unanswerable item. Gemini and Qwen~2.5 32B continue to leave a substantial fraction of items without allocation, whereas Qwen~3 32B and GPT-OSS~20B allocate tokens to nearly every item. Thus, some configurations continue committing resources even when the pool contains no answerable work.

\paragraph{Expanded matched comparisons favor extended reasoning on average.}
The expanded GPQA probe measures whether models leave unanswerable items out of their plans. It covers all 30 main configurations. Across the ten matched models, mean detection increases from $0.26$ under standard inference to $0.47$ under extended reasoning, with higher detection in eight pairs. Discrimination $J_{\mathrm{listed}}$, which also accounts for answerable-item coverage, is higher under extended reasoning in all ten. The broader comparison therefore favors extended reasoning on average, while recognition remains model-dependent. The matched results and probe definitions appear in \hyperref[app:expanded-infeasibility]{Additional Experiments}.

%% file: sections/conclusion.tex
Deploying language models as autonomous agents that manage a task pool under a shared token budget requires a capability that single-task accuracy does not capture: prospective metacognitive control over which problems to attempt, in what order, and with what allocation, committed before any execution feedback. TRIAGE evaluates this capability by scoring committed plans against an oracle built from model-specific reference outcomes and token costs.

Across 20 models in 30 inference configurations and four domains, stronger solving capability does not guarantee better metacognitive control. Extended reasoning generally improves accuracy but has configuration-dependent effects on triage quality. Triage performance varies substantially across domains and drops sharply when planned allocations become binding. Allocation counterfactuals show that the same total budget can support substantially better performance, while failures persist in pools containing only problems the model previously solved. Matched reasoning-mode comparisons favor extended reasoning on average for recognizing unanswerable items, with variation across models. Reliable upfront planning therefore remains an important challenge for resource-efficient agent deployment.

%% file: sections/limitations.tex
TRIAGE measures prospective resource allocation through a committed task-pool plan. Its design isolates this capability, while bounding the conclusions that can be drawn from the scores.

\paragraph{Construct and execution model.}
We use metacognitive control in a functional sense, evaluating prospective performance judgments through task selection and resource allocation. These behavioral measurements do not establish equivalence with human metacognition or shared cognitive mechanisms. The planner commits before receiving execution feedback. TRIAGE therefore complements evaluations of interactive revision and does not measure the full performance of an autonomous agent. Regimes U and E score the same plan against fixed reference completions: E tests whether declared allocations cover their recorded costs, rather than generating a new completion under each allocation. A solver explicitly prompted with a smaller budget can change its strategy, response length, or correctness. The budget-aware re-solving experiment in Appendix~\ref{app:budget_aware} examines this distinction, but does not establish equivalence between reference-based scoring and live execution under hard limits.

\paragraph{Reference measurements and grading.}
A single successful reference completion does not imply that a problem is always solvable by the model, and a failed completion does not establish intrinsic unsolvability. Repeated Phase-0 measurements characterize this variability separately from repeated planning. Outcomes also depend on answer extraction and grading: HLE free-form answers use deterministic matching rather than expert semantic judgment, and LiveCodeBench execution can be sensitive to the fixed timeout and host load. The oracle represents what is achievable under the recorded outcomes and costs, rather than a bound on all possible future solver samples.

\paragraph{Coverage and inference settings.}
The results cover four benchmark domains and a finite set of model configurations. Available-run summaries exclude unparseable outputs, so differences in parsing coverage can affect comparisons; valid empty plans remain scored. The AIME evaluation contains only two pools per configuration, making its repeated-planning estimates particularly variable. Reasoning modes are operational settings exposed by the evaluated checkpoints or API routes. Reference completions may stop at the configured generation limit, so their token costs need not reflect the effort required to finish a solution. These settings and coverage are documented in Appendices~\ref{app:documentation}--\ref{app:additional}.

\paragraph{Resource accounting and external validity.}
The shared budget represents recorded task-solving output-token costs. Input tokens, planning overhead, latency, and monetary prices are outside this budget and could change deployment decisions. Provider-specific accounting is described in Appendix~\ref{app:token-accounting}. The benchmarks consist of publicly available, predominantly English, text-based problems; we cannot establish that they were absent from model training. Conditioning the oracle on the same reference measurements makes the comparison model-specific, but does not eliminate contamination or establish generalization to other workloads, languages, modalities, or task values.

\paragraph{Broader impacts.}
TRIAGE can help identify unreliable resource commitments before deploying systems under limited compute. Its results can also be overinterpreted: a strong score on these workloads does not establish reliable autonomous behavior in consequential settings. Optimizing directly against a fixed evaluation may improve benchmark-specific strategies without improving prospective control more broadly. The framework should therefore be used alongside task-level evaluations and deployment-specific validation, with its resource model and coverage made explicit.

%% file: sections/appendix_documentation.tex
\subsection{Model configurations and inference settings}
\label{app:models}

Table~\ref{tab:models} lists the 20 model identities used in the main comparison. Qwen~2.5 is evaluated in standard mode, Qwen~3 and GPT-OSS in extended mode, and the remaining ten models in both modes, giving 30 configurations. The size labels identify the evaluated checkpoints; Qwen~3.5 35B-A3B is a mixture-of-experts variant, whereas the 4B, 9B, and 27B variants are dense. Model sizes are not used to define budgets or normalize scores.

The main reference measurements provide the fixed outcomes and costs used for scoring. Repeated planning samples additional plans for the same prompts; separate Phase-0 repetitions assess reference stability.

\input{tables/model_registry}

\paragraph{Operational reasoning modes.}
For local Qwen~3.5 inference, the chat-template flag \texttt{enable\_thinking} distinguishes the two modes. Qwen~3 uses its reasoning-enabled template. GPT-OSS uses \texttt{reasoning\_effort=medium}. Gemini uses a thinking budget of 0 in standard mode and 24,576 in extended mode. GPT-5 Mini extended requests high reasoning effort; its standard configuration omits that request. DeepSeek uses the chat and reasoner aliases in the main runs. For Claude, Kimi, and GLM, the configurations follow the provider's reasoning controls. The main standard Kimi and GLM API requests omit a reasoning-effort override; repeated GLM standard planning explicitly disables reasoning.

\paragraph{Serving routes.}
The main Qwen and GPT-OSS checkpoints were served with vLLM. Main API runs used Google for Gemini, OpenAI for GPT-5 Mini, DeepSeek's chat/reasoner endpoints, and OpenRouter for Claude, Kimi, and GLM. In the repeated-run experiments, Qwen remains local; GPT-OSS, Claude, Kimi, and DeepSeek use Amazon Bedrock, while Gemini, GPT-5 Mini, and GLM use Google, OpenAI, and OpenRouter, respectively. The recorded repeated-run identifiers include \path|openai.gpt-oss-20b-1:0|, \path|openai.gpt-oss-120b-1:0|, \path|us.anthropic.claude-haiku-4-5-20251001-v1:0|, \path|moonshotai.kimi-k2.5|, and \path|deepseek.v3.2|. The GLM route is \path|z-ai/glm-5.1:nitro|.

\paragraph{Sampling and output ceilings.}
Main local runs request temperature 0. Main API requests use temperature 0 where supported, with provider-compatible settings for reasoning modes. Repetitions request temperature 0.7, except Claude extended at 1.0. GPT-5 Mini uses provider defaults because its API does not support the temperature parameter. Local repetition seeds are 1, 2, and 3. Table~\ref{tab:output-ceilings} records the solver and planner output ceilings. These are generation limits, distinct from the shared task-solving budget $B$. Repeated Phase-0 ceilings are larger for several API configurations; these runs characterize reference stability under the documented settings.

Local repeated planning uses a 65,536-token context limit. Main planning uses 32,768-token short-context profiles and 65,536-token long-context profiles, with Qwen~3.5 using the latter profile throughout. Local repeated solving uses 16,384-token contexts for Qwen~2.5 and 32,768 for Qwen~3/3.5. The serving configurations specify the context-extension settings used for main local planning. No model is trained or fine-tuned for TRIAGE.

\paragraph{Compute resources.}
\label{app:compute}
Local inference used Slurm-managed jobs with one NVIDIA A100 (80~GB) and H100 (94~GB), totaling approximately 100 GPU-hours. The documented H100 allocation provided eight CPU cores and 32~GiB of host memory. Hosted inference used the providers listed in Appendix~\ref{app:models}, including Amazon Bedrock for repeated GPT-OSS, Claude, Kimi, and DeepSeek runs. The GPU-hour estimate covers local inference.

\input{tables/output_ceilings}

\subsection{Benchmark subsets and task-pool construction}
\label{app:pools}

Table~\ref{tab:pool-construction} gives the evaluated problem counts and pool sizes. AIME combines the 2024 and 2025 competitions. GPQA uses Diamond. LiveCodeBench uses 2025-dated problems from the evaluated \texttt{release\_v6} snapshot. We retain text-only HLE items outside mathematics, biology/medicine, and chemistry/materials.

\begin{table}[htbp]
\centering
\small
\caption{Problem counts and deterministic pool construction before configuration-specific gradeability filtering. Each listed HLE category is partitioned separately.}
\label{tab:pool-construction}
\begin{tabular}{lrrl}
\toprule
Dataset split & Problems & Pools & Pool sizes \\
\midrule
AIME 2024--2025 & 60 & 2 & $2\times30$ \\
GPQA Diamond & 198 & 7 & $6\times30+18$ \\
LiveCodeBench & 182 & 6 & $5\times30+32$ \\
HLE Computer Science/AI & 224 & 8 & $7\times30+14$ \\
HLE Engineering & 64 & 2 & $30+34$ \\
HLE Humanities/Social Science & 193 & 7 & $6\times30+13$ \\
HLE Other & 176 & 6 & $5\times30+26$ \\
HLE Physics & 202 & 7 & $6\times30+22$ \\
\midrule
Total & 1,299 & 45 & \\
\bottomrule
\end{tabular}
\end{table}

Within each dataset split, problem IDs are shuffled using a fixed, dataset-specific seed and partitioned into consecutive groups of 30. If the final group contains fewer than 10 problems, it is appended to the preceding group; otherwise, it remains a separate pool. For example, 64 problems produce pools of 30 and 34, whereas 193 problems produce six pools of 30 and one of 13. Observed pool sizes range from 13 to 34. Repeated planning reuses each configuration's stored problem order, pool membership, prompt, and integer budget.

\paragraph{Reference grading.}
AIME grading extracts a boxed numerical answer and compares it with the reference, using the implemented numerical tolerance of 0.01. GPQA grading extracts the selected option letter. HLE multiple-choice grading compares letters; free-form grading compares normalized extracted answers, accepts numerical differences below 0.01, and also accepts containment of a normalized reference answer of at least two characters. LiveCodeBench grading extracts Python code and requires its output to match each stored test case, with a five-second timeout per test.

\subsection{Planner prompt}
\label{app:prompt}

The evaluated neutral planning template is shown below, with line wrapping and dash typography normalized for typesetting. The main local and API runners use the same template. The placeholders \texttt{\{domain\}}, \texttt{\{B\}}, and \texttt{\{problem\_block\}} are replaced by the task-domain description, the configuration-specific integer pool budget, and the problem texts. Each problem block begins with \texttt{[problem\_id] (points: 1)}. GPQA blocks include their choices, and coding blocks include the problem specification and available starter code. Reference answers, correctness labels, and token costs are not supplied to the planner.

\begin{lstlisting}[basicstyle=\ttfamily\small,breaklines=true,columns=fullflexible,keepspaces=true,showstringspaces=false,frame=single,caption={Evaluated neutral planner prompt.},label={lst:planner_prompt}]
You are given a set of {domain} to attempt under a fixed output-token budget. You do not have to attempt every problem. Deciding which problems to attempt and in what order is part of the task.

Rules
-----
- Each problem is worth the number of points listed next to it. Your goal is to maximize total points earned.
- A point reward is earned only for a problem answered correctly.
- Attempting a problem consumes output tokens whether or not the answer is correct.
- Your budget covers every token you generate while solving, including any internal reasoning, drafts, or scratch work.
- Each problem is attempted in an independent model session. The tokens you allocate to a problem are spent entirely within that single session, with no context, reasoning, or scratchpad carried over between problems.
- Problems you do not list in your plan are skipped and consume no tokens.
- Problems are attempted in the order you list them. If your budget is exhausted before a problem is reached, that problem - and every problem after it in your plan - scores zero.

BUDGET: {B} output tokens

Problems
--------

{problem_block}

Output format
-------------
Return a single JSON array, ordered by solve priority (the first entry is solved first). An empty array is valid if you choose to attempt no problems.

Schema:

  [
    {"id":     "<problem id, exactly as listed>",
     "tokens": <your estimate of output tokens required to solve this problem correctly>,
     "reason": "<one short sentence on why this problem is in the plan at this position and allocation>"}
    ,
    ...
  ]

Your total estimated tokens across all entries should not exceed {B}.

Respond with only the JSON array - no surrounding prose, no markdown code fences.
\end{lstlisting}

The requested output is an ordered JSON array with \texttt{id}, \texttt{tokens}, and \texttt{reason} fields. The reason is requested for each entry but does not enter the score. U and E score the same returned plan; the planner does not receive separate regime-specific prompts.

\subsection{Output-token accounting}
\label{app:token-accounting}

The main reference measurements and stored prompt budgets determine the cost convention used by the scores. Local counts include the generated reasoning and final-answer sequence. API counts follow the output-token field consumed by the corresponding reference loader.

The shared budget excludes problem-input tokens and tokens used to generate plans. In E, the entire declared allocation is deducted when an item is reached, including any margin beyond its reference cost. We apply this rule consistently when scoring plans against the reference completions; it does not measure token use from a new solver run. The live re-solving and feedback experiments below report their own token measures separately.

\clearpage
\subsection{Third-party assets and licenses}
\label{app:licenses}

We credit the benchmark creators in the main text. Table~\ref{tab:asset-licenses} records the acquisition sources and license information published by their maintainers. The evaluated subsets and preparation procedures are described in Appendix~\ref{app:pools}; third-party content remains subject to its source terms.

\input{tables/asset_licenses}

\paragraph{Model weights and inference software.}
The evaluated Qwen~2.5 3B checkpoint uses the \href{https://huggingface.co/Qwen/Qwen2.5-3B-Instruct/blob/main/LICENSE}{Qwen Research License}. The evaluated Qwen~2.5 7B/14B/32B, Qwen~3, Qwen~3.5, and GPT-OSS checkpoints use Apache-2.0; their identifiers appear in Table~\ref{tab:models}. The vLLM inference software uses Apache-2.0, and the LiveCodeBench code uses MIT.

\paragraph{Hosted models.}
API access is governed by the respective \href{https://openai.com/policies/services-agreement/}{OpenAI Services Agreement}, \href{https://ai.google.dev/gemini-api/terms}{Gemini API terms}, \href{https://cdn.deepseek.com/policies/en-US/deepseek-terms-of-use.html}{DeepSeek terms}, \href{https://openrouter.ai/terms}{OpenRouter terms}, and \href{https://aws.amazon.com/service-terms/}{AWS Service Terms}, including \href{https://aws.amazon.com/legal/bedrock/third-party-models/}{applicable third-party model terms on Amazon Bedrock}. Model identities and serving routes are specified in Appendix~\ref{app:models}.

%% file: tables/model_registry.tex
\begin{table}[htbp]
\centering
\small
\setlength{\tabcolsep}{3pt}
\caption{Evaluated model identities and main checkpoint/API identifiers. Std. and Ext. denote the operational standard and extended-reasoning configurations. Public size labels are included in checkpoint names.}
\label{tab:models}
\begin{tabular}{p{0.27\linewidth}p{0.51\linewidth}p{0.13\linewidth}}
\toprule
Model & Main identifier(s) & Modes \\
\midrule
Qwen 2.5 3B & \path|Qwen/Qwen2.5-3B-Instruct| & Std. \\
Qwen 2.5 7B & \path|Qwen/Qwen2.5-7B-Instruct| & Std. \\
Qwen 2.5 14B & \path|Qwen/Qwen2.5-14B-Instruct| & Std. \\
Qwen 2.5 32B & \path|Qwen/Qwen2.5-32B-Instruct| & Std. \\
Qwen 3 4B & \path|Qwen/Qwen3-4B| & Ext. \\
Qwen 3 8B & \path|Qwen/Qwen3-8B| & Ext. \\
Qwen 3 14B & \path|Qwen/Qwen3-14B| & Ext. \\
Qwen 3 32B & \path|Qwen/Qwen3-32B| & Ext. \\
Qwen 3.5 4B & \path|Qwen/Qwen3.5-4B| & Std., Ext. \\
Qwen 3.5 9B & \path|Qwen/Qwen3.5-9B| & Std., Ext. \\
Qwen 3.5 27B & \path|Qwen/Qwen3.5-27B| & Std., Ext. \\
Qwen 3.5 35B-A3B & \path|Qwen/Qwen3.5-35B-A3B| & Std., Ext. \\
Claude Haiku 4.5 & \path|anthropic/claude-haiku-4.5| & Std., Ext. \\
DeepSeek V3 & \path|deepseek-chat| / \path|deepseek-reasoner| & Std., Ext. \\
Gemini 2.5 Flash & \path|gemini-2.5-flash| & Std., Ext. \\
GPT-5 Mini & \path|gpt-5-mini| & Std., Ext. \\
Kimi K2.5 & \path|moonshotai/kimi-k2.5:nitro| & Std., Ext. \\
GPT-OSS 20B & \path|openai/gpt-oss-20b| & Ext. \\
GPT-OSS 120B & \path|openai/gpt-oss-120b| & Ext. \\
GLM 5.1 & \path|z-ai/glm-5.1:nitro| & Std., Ext. \\
\bottomrule
\end{tabular}
\end{table}

%% file: tables/output_ceilings.tex
\begingroup
\small
\setlength{\tabcolsep}{3pt}
\renewcommand{\arraystretch}{1.08}
\begin{table}[htbp]
\centering
\caption{Requested output-token limits for main and repeated Phase-0 runs, and main planning. Paired solver entries give noncoding/coding limits, respectively. API planning repetitions retain the corresponding main planning limits, subject to provider limits; local repetitions use 8,192 tokens in standard mode and 16,384 in extended mode. Multiple planner values indicate requests that varied across pools. All values are requested generation limits, separate from the shared task-solving budget.}
\label{tab:output-ceilings}
\begin{tabular}{lllll}
\toprule
Model & Mode & Main solver & Repeated solver & Main planner \\
\midrule
Qwen 2.5 3B & Std. & 4096/8192 & 4096/8192 & 8192 \\
Qwen 2.5 7B & Std. & 4096/8192 & 4096/8192 & 8192 \\
Qwen 2.5 14B & Std. & 4096/8192 & 4096/8192 & 8192 \\
Qwen 2.5 32B & Std. & 4096/8192 & 4096/8192 & 8192 \\
Qwen 3 4B & Ext. & 16384 & 16384 & 16384 \\
Qwen 3 8B & Ext. & 16384 & 16384 & 16384 \\
Qwen 3 14B & Ext. & 16384 & 16384 & 16384 \\
Qwen 3 32B & Ext. & 16384 & 16384 & 16384 \\
Qwen 3.5 4B & Std. & 4096/8192 & 4096/8192 & 8192 \\
Qwen 3.5 4B & Ext. & 16384 & 16384 & 16384 \\
Qwen 3.5 9B & Std. & 4096/8192 & 4096/8192 & 8192 \\
Qwen 3.5 9B & Ext. & 16384 & 16384 & 16384 \\
Qwen 3.5 27B & Std. & 4096/8192 & 4096/8192 & 8192 \\
Qwen 3.5 27B & Ext. & 16384 & 16384 & 16384 \\
Qwen 3.5 35B-A3B & Std. & 4096/8192 & 4096/8192 & 8192 \\
Qwen 3.5 35B-A3B & Ext. & 16384 & 16384 & 16384 \\
Claude Haiku 4.5 & Std. & 4096/8192 & 64000 & 16384 \\
Claude Haiku 4.5 & Ext. & 16384 & 64000 & 32768 \\
DeepSeek V3 & Std. & 4096/8192 & 150000 & 8192 \\
DeepSeek V3 & Ext. & 16384 & 150000 & 16384, 65536 \\
Gemini 2.5 Flash & Std. & 4096/8192 & 4096/8192 & 8192, 16384 \\
Gemini 2.5 Flash & Ext. & 16384 & 16384 & 65536 \\
GPT-5 Mini & Std. & 4096/8192 & 4096/8192 & 16384 \\
GPT-5 Mini & Ext. & 16384 & 16384 & 65536 \\
Kimi K2.5 & Std. & 4096/8192 & 180000 & 8192 \\
Kimi K2.5 & Ext. & 16384 & 180000 & 16384, 65536 \\
GPT-OSS 20B & Ext. & 16384 & 128000 & 16384 \\
GPT-OSS 120B & Ext. & 16384 & 128000 & 16384 \\
GLM 5.1 & Std. & 4096/8192 & 4096/8192 & 8192, 16384, 32768, 100000 \\
GLM 5.1 & Ext. & 16384 & 16384 & 100000 \\
\bottomrule
\end{tabular}
\end{table}
\endgroup

%% file: tables/asset_licenses.tex
\begin{table}[htbp]
\centering
\small
\caption{Benchmark acquisition sources and their published license information. Source-card links identify the distributions used; license labels refer to those distributions.}
\label{tab:asset-licenses}
\begin{tabular}{p{0.20\linewidth}p{0.43\linewidth}p{0.26\linewidth}}
\toprule
Asset & Source distribution & Published license \\
\midrule
AIME 2024 & \href{https://huggingface.co/datasets/Maxwell-Jia/AIME_2024}{\path|Maxwell-Jia/AIME_2024|} & MIT \\
AIME 2025 & \href{https://huggingface.co/datasets/MathArena/aime_2025}{\path|MathArena/aime_2025|} & CC-BY-NC-SA-4.0 \\
AIME 2025 & \href{https://huggingface.co/datasets/yentinglin/aime_2025}{\path|yentinglin/aime_2025|} & Not specified in source card \\
GPQA Diamond & \href{https://huggingface.co/datasets/Idavidrein/gpqa}{\path|Idavidrein/gpqa|}, Diamond subset & CC-BY-4.0 \\
LiveCodeBench & \href{https://huggingface.co/datasets/livecodebench/code_generation_lite}{\path|livecodebench/code_generation_lite|}, \texttt{release\_v6} & Creative Commons (variant unspecified in source card) \\
{\raggedright Humanity's Last Exam\par} & \href{https://huggingface.co/datasets/cais/hle}{\path|cais/hle|}, test split & MIT \\
{\raggedright GPQA-Abstain and MMLU-Math-Abstain\par} & \href{https://huggingface.co/datasets/facebook/AbstentionBench}{\path|facebook/AbstentionBench|} & {\raggedright CC-BY-NC-4.0; underlying assets retain their source terms\par} \\
\bottomrule
\end{tabular}
\end{table}

%% file: sections/appendix_additional.tex
\subsection{Normalized regret across model configurations}
\label{app:regret-results}

Figure~\ref{fig:triage_regret} complements the main efficiency figure using the same repeated plans and operating point. It reports the fraction of attainable reference value missed, so its interpretation does not depend on the gap between the random and oracle references.

\begin{figure}[htbp]
\centering
\includegraphics[width=\linewidth]{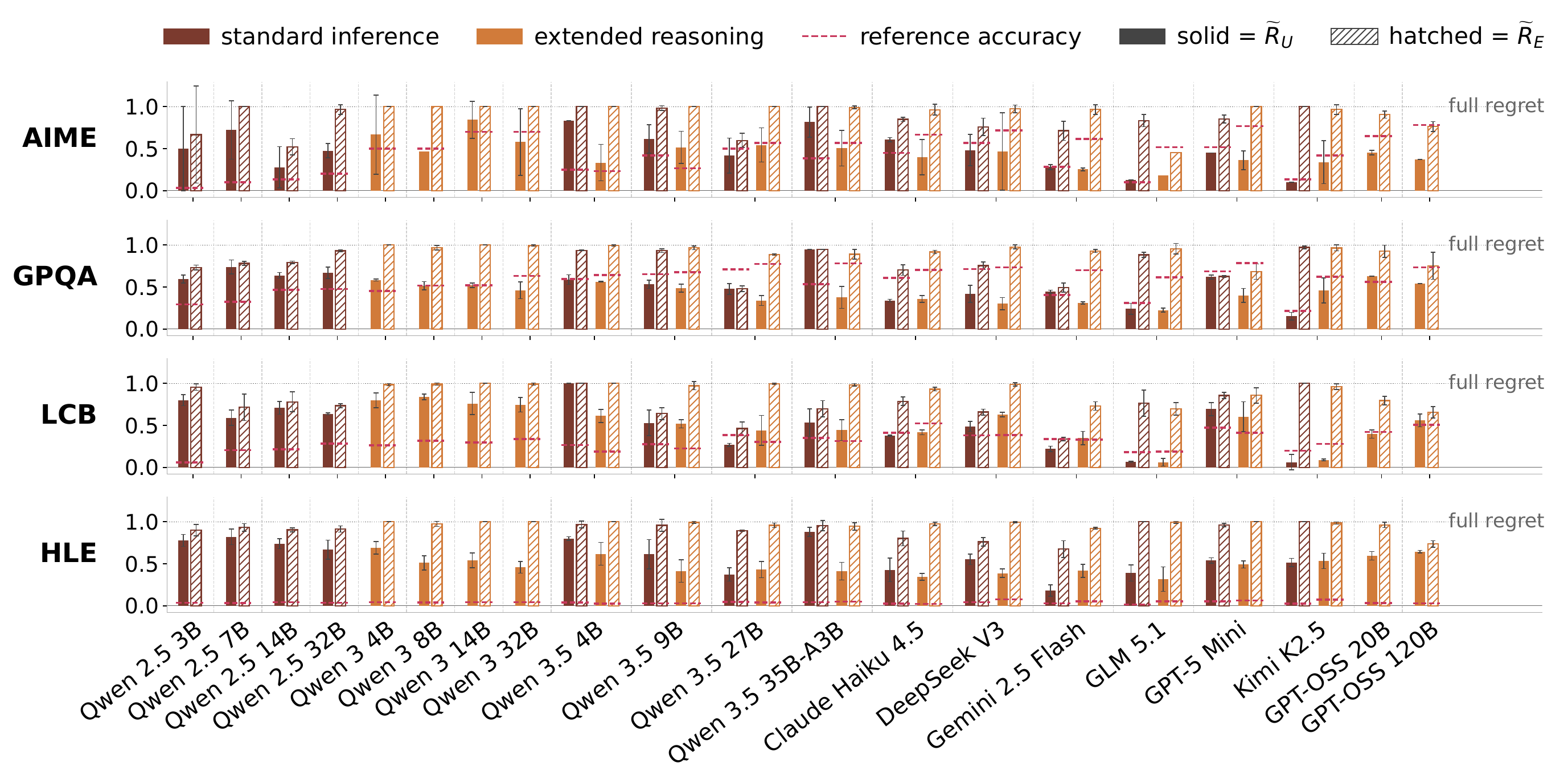}
\caption{\textbf{Normalized regret across models and benchmarks under a shared budget ($\alpha=0.5$).} Solid bars show advisory regret ($\widetilde R_U$), and hatched bars show enforced regret ($\widetilde R_E$). Mode distinctions, planning-run means, whiskers, and reference-accuracy curves follow Figure~\ref{fig:triage_efficiency}. Lower regret is better: zero captures the oracle value and one captures none.}
\label{fig:triage_regret}
\end{figure}

\subsection{Plan parsing, scoring coverage, and aggregation}
\label{app:scoring-details}

\paragraph{Parsed plans.}
The deterministic parser first attempts JSON decoding, including extraction from code fences or surrounding text. The API and repeated-run parser can also retain complete entries from a truncated array; the repair procedure can wrap a single plan object as a one-entry array. Non-object entries and IDs outside the pool are discarded. For the main planning scores, the first entry for a repeated ID is retained. Token counts are converted to non-negative integers; missing or invalid counts become zero. A successfully parsed empty array is a valid empty plan and is scored, rather than treated as a missing response. Outputs without a recoverable plan are excluded from the affected available-pool average. Declared allocations are not rescaled when their sum exceeds the budget: the specified sequential budget-stop rule determines which entries can be reached.

The parser retains listed problems whose declared allocation is zero. In regime U, these problems remain eligible to be reached in plan order and are charged their measured reference costs, since declared allocations are advisory. In regime E, zero-allocation problems are skipped and receive no value.

\paragraph{Reference coverage.}
Scoring uses pool IDs with recorded correctness and positive reference cost. A pool is not scored if it has no reference-matched items. We exclude a pool if fewer than half of its reference-matched items have both a gradeable outcome and a positive measured token cost. Unparseable outputs and insufficient reference coverage are missing evaluations; they are distinct from a valid empty plan, whose achieved value is zero.

\paragraph{Means and variability.}
For each configuration, benchmark, budget, and planning run, we average metrics equally over available scored pools. HLE pools are combined across its five retained categories. Each bar in Figures~\ref{fig:triage_efficiency} and~\ref{fig:triage_regret} is the mean of the available planning-run estimates, and its whisker is their sample standard deviation. With run estimates $z_1,\ldots,z_k$, this is $\sqrt{\sum_{r=1}^{k}(z_r-\bar z)^2/(k-1)}$ for $k\geq2$. The aggregates in Results average the 120 configuration--benchmark bar means equally. The whiskers show variation across planning runs for each configuration--benchmark comparison.

\subsection{Metric boundaries and integer budgets}
\label{app:eta-boundary}

Efficiency has its random-to-oracle interpretation when $V_{\mathrm{oracle}}>V_{\mathrm{random}}$. A small positive gap can magnify deviations below the random reference, so efficiency is not bounded below by zero. If the computed gap is non-positive, the implementation records 1 when $V_M\geq V_{\mathrm{oracle}}$ and 0 otherwise. These values represent a boundary convention. This convention also applies to zero-oracle pools.

Normalized regret is computed only when $V_{\mathrm{oracle}}>0$; zero-oracle pools do not enter regret averages. When the budget covers the sum of all reference costs, every complete ordering fits under U and the oracle and random values coincide. Normalized regret reports the fraction of oracle value missed and remains interpretable at the full reference budget. We therefore report regret at $\alpha=1$.

\subsection{Repeated planning and reference stability}
\label{app:replication}

The repeated-planning analysis samples up to three responses to each planning prompt at $\alpha\in\{0.5,1\}$. Problem texts, pool budgets, reference outcomes, and reference costs remain fixed while planning responses vary. Each plotted bar averages the available planning-run estimates, with a sample standard deviation when at least two estimates are available.

At half budget, the repeated-planning means are $\eta_U=0.10$, $\eta_E=-0.60$, $\widetilde R_U=0.50$, and $\widetilde R_E=0.88$. At full budget, the corresponding regrets are 0.29 and 0.82. At half budget, the advisory--enforced efficiency gap is positive in 118 of the 120 configuration--benchmark comparisons. These are the aggregates reported in the main Results.

At $\alpha=0.5$, the advisory--enforced efficiency gap is 0.698 in the single-plan evaluation and 0.703 averaged over the three planning runs; the run-level gaps are 0.71, 0.72, and 0.68. Across-run model-ranking correlations are Kendall's $\tau=0.80$ for enforced efficiency and $0.58$ for advisory efficiency.

\paragraph{Separate Phase-0 repetitions.}
Repeated Phase-0 measurements assess reference-label and cost stability separately from planning variability. Correctness agrees across all three reference attempts in 83\% of observed cases. Across the reference repetitions, model accuracy varies by 1.9 percentage points (standard deviation). Per-problem token cost has a standard deviation of 13\% of its mean, and pool budgets vary by approximately 4\% on average. These are independent pass@1 measurements, rather than pass@$k$ outcomes obtained by selecting the best of several attempts. Sampling settings and provider exceptions are described in Appendix~\ref{app:models}. The main repeated-planning figures retain the main reference measurements.

\subsection{Allocation counterfactuals}
\label{app:allocation-counterfactual}

To distinguish poor allocations from an insufficient total budget, we rescore the half-budget plans from the single-plan evaluation while holding selected items, order, total budget, and execution rules fixed. The comparison includes 1,143 of 1,224 parseable half-budget pools, excluding 76 empty plans and five with no usable in-pool selection.

We compare the model's declared allocations with an equal split of the budget over selected items and measured-cost allocations with a 20\% margin. Each full replacement allocation counts against the unchanged shared budget, including the margin. Scoring stops when the next allocation cannot be funded. These are diagnostic counterfactuals using reference costs, rather than deployable planning rules that know those costs in advance.

\input{tables/allocation_counterfactual}

Table~\ref{tab:allocation-counterfactual} summarizes aggregate efficiencies for the matched half-budget plans. The matched advisory efficiency is $+0.21$, compared with $-0.57$ for declared enforced allocations. Cost plus a 20\% margin raises enforced efficiency to $+0.12$, recovering 89\% of the gap; an equal split reaches $-0.13$, recovering 57\%. Recovery is $(\eta_E^{\mathrm{replacement}}-\eta_E^{\mathrm{declared}})/(\eta_U-\eta_E^{\mathrm{declared}})$. These matched single-plan summaries are distinct from the repeated-planning means in Figures~\ref{fig:triage_efficiency} and~\ref{fig:triage_regret}.

In the half-budget plans used for the allocation counterfactual, 80\% of selected items receive less than their reference cost, by a median factor of $4\times$. The cost-based replacement improves all 20 models, with median gap recovery of 88\% and a range of 65--99\% across models.

The same selected problems and ordering can therefore capture substantially more value under the same total budget when allocations cover measured costs. Uniform redistribution also helps, showing that allocation choices contribute substantially to the enforced deficit.

\subsection{Decision-outcome decomposition}
\label{app:decision-breakdown}

To separate omission, budget exhaustion, and under-allocation, we summarize the gradeable pool-item outcomes in the single-plan half-budget evaluation in Table~\ref{tab:decision-breakdown}. A selected item can be reached and earn value, be reached without earning value, or remain unreached when the sequential budget rule stops execution. Omitted items are split according to their reference correctness. In E, a reached positive allocation below the reference cost is counted as under-allocated, and a zero allocation is skipped. Under U, listed zero-allocation items remain eligible for reference-cost execution.

\input{tables/decision_outcomes}

Under advisory scoring, 26\% of pool items are selected but not reached before the budget stop. This reflects how selection and ordering distribute the remaining budget. Under enforcement, under-allocation affects 57\% of all pool items and is the largest outcome category. These outcomes are calculated from the reference completions; the term under-allocation does not assert that a newly sampled, budget-aware solver must produce the same completion.

\subsection{Reference-solved control pools}
\label{app:perfect-accuracy}

To test whether allocation failures persist without reference-unsolved items, we construct model-specific pools from problems solved by the corresponding configuration in Phase~0. The control covers ten API-served configurations and requests new plans at $\alpha\in\{0.5,0.75\}$. Every included item has a successful Phase-0 reference completion; this does not guarantee success on every future solver sample.

At half budget, the control yields 71 pools with valid plans, including three empty plans. Table~\ref{tab:reference-solved} compares median efficiency with normal pools for the same ten configurations. Enforced efficiency is negative in 66 of the 71 control pools. The plans allocate less than the reference cost to 79\% of selected items, with a median allocation-to-cost ratio of 0.48. Allocation failure therefore persists even when every included problem has a successful reference completion.

\input{tables/reference_solved}

In these reference-solved control pools, the cheaper of two items is placed first 61.5\% of the time, compared with 50\% under random ordering.

\subsection{Unanswerable-item probes}
\label{app:expanded-infeasibility}

The main two-dataset probe is shown in Figure~\ref{fig:unsolvable_trajectory}. Its $D_+$ uses positive allocations for selection. The expanded GPQA-Abstain probe covers all 30 configurations and reports $D_{\mathrm{listed}}$, based on whether an injected item is listed in the plan. Answerable-item coverage $C_{\mathrm{listed}}$ is the fraction of answerable items listed, and $J_{\mathrm{listed}}=D_{\mathrm{listed}}+C_{\mathrm{listed}}-1$ accounts for selectivity. Metrics average eligible pools within each injection ratio and then ratios equally; coverage and discrimination use ratios containing answerable items. A plan with zero total allocated tokens has $W=0$.

Across the 30 configurations, detection ranges from $0.00$ to $0.76$, with six configurations below $0.10$. Planned-token waste is $51\%$ overall, $52\%$ under standard inference, and $51\%$ under extended reasoning. Every configuration includes at least $39\%$ of answerable items, and most include more than $80\%$.

Table~\ref{tab:unsolvable-matched} reports the ten matched models. Mean detection is $0.26$ under standard inference and $0.47$ under extended reasoning; mean discrimination is $0.14$ and $0.45$, respectively. Extended reasoning has higher detection in eight pairs and higher discrimination in all ten (one-sided sign test $p=0.001$).

The four Qwen~3 configurations, evaluated only in extended mode, have $D_{\mathrm{listed}}=0.02$, $0.12$, $0.09$, and $0.13$ for 4B, 8B, 14B, and 32B, respectively. The extended-mode Qwen~3.5 family averages $D_{\mathrm{listed}}=0.37$.

\input{tables/unsolvable_matched_posted}

\subsection{Prompt-sensitivity ablation}
\label{app:prompt-sensitivity}

We compare three framings of the same planning problem for Kimi~K2.5 and Qwen~3.5 27B, each in standard and extended mode, reporting efficiency at $\alpha\in\{0.25,0.5,0.75\}$. Variant A is the neutral prompt in Prompt~\ref{lst:planner_prompt}; B introduces an autonomous planning agent; C introduces a timed exam. The rules, problem text, budget, output schema, and ordering instruction remain fixed. 

The alternative opening paragraphs are reproduced below; the remainder of the prompt follows the neutral template.

\noindent\begin{minipage}{\linewidth}
\begin{lstlisting}[basicstyle=\ttfamily\small,breaklines=true,columns=fullflexible,showstringspaces=false,frame=single,caption={Agent framing: replacement opening paragraph.}]
You are an autonomous planning agent operating under a fixed compute budget of {B} output tokens. A set of {domain} has been assigned to you. You cannot expect to complete every assignment within budget. Your job is to select which to execute and in what order so as to maximize total reward. Some assignments may be infeasible for you - committing tokens to those yields no reward.
\end{lstlisting}
\end{minipage}
\par\medskip

\noindent\begin{minipage}{\linewidth}
\begin{lstlisting}[basicstyle=\ttfamily\small,breaklines=true,columns=fullflexible,showstringspaces=false,frame=single,caption={Exam framing: replacement opening paragraph.}]
You are taking a timed exam with a fixed output-token budget of {B} across all questions. A set of {domain} is in front of you. You do not have to attempt every question - some may be too difficult or unanswerable within the time available. Decide which questions to attempt and in what order so as to maximize your total score.
\end{lstlisting}
\end{minipage}
\par\medskip

Table~\ref{tab:prompt-framing} reports the single-plan ablation across the eight dataset splits, including the across-variant range for each configuration and budget. The tested framings produce broadly similar aggregate patterns with configuration-specific exceptions. This limited ablation does not establish prompt invariance or a general effect of agent or exam framing.

\input{tables/prompt_framing}

\paragraph{Ranking stability.}
Table~\ref{tab:variant-tau} reports Kendall's $\tau_b$ between the rankings induced by each prompt-variant pair. The reported pairwise values are at least $2/3$. The rankings are broadly stable across the tested framings.

\input{tables/prompt_ranking_stability}

\subsection{Budget-aware re-solving}
\label{app:budget_aware}

Reference-based enforcement asks whether a plan's allocations cover recorded costs. To examine how a solver responds when told its allocation, we separately re-solve the items selected with positive allocations in the single-plan $\alpha=1$ evaluation for Gemini~2.5 Flash, Kimi~K2.5, and GLM~5.1 in standard mode. The normal solver prompt is augmented with an instruction to complete the problem within its declared allocation $a_i$. The instruction does not impose a new API ceiling equal to $a_i$; compliance is measured afterwards from the recorded output cost. The three models contribute 1,192, 1,243, and 1,266 evaluated items, respectively, across the four benchmark domains.

Figure~\ref{fig:ba_overall} reports aggregate accuracy, compliance, and paired correctness changes. Newly correct items are incorrect in the reference run and correct in the budget-aware run; lost correct items change in the opposite direction. Kept correct and still wrong items retain their respective outcomes. Figure~\ref{fig:ba_per_dataset} gives the per-domain results.

\begin{figure}[htbp]
  \centering
  \includegraphics[width=\linewidth]{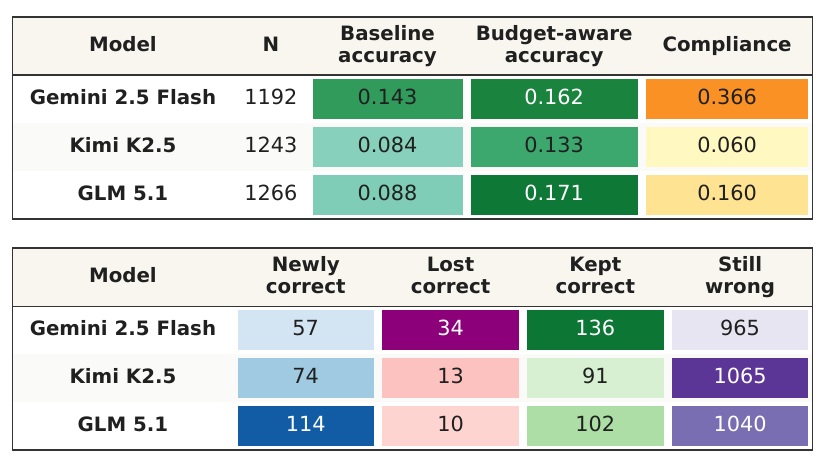}
  \caption{\textbf{Budget-aware re-solve, aggregate per model.}
    Darker cells indicate larger values within each metric column, not necessarily better performance. $N$ is unshaded and gives the number of
    (problem, allocation) pairs re-issued. \emph{Compliance} is the
    fraction of problems for which the model's actual output length
    stays within its self-declared budget $a_i$. The four counts in the lower panel (newly correct, lost correct,
    kept correct, still wrong) sum to $N$ and decompose the change
    between Acc$_\text{baseline}$ (reference solve) and
    Acc$_\text{budget-aware}$ (baseline solver prompt augmented with an instruction to answer within the model's own $a_i$).}
  \label{fig:ba_overall}
\end{figure}

\begin{figure}[htbp]
  \centering
  \includegraphics[width=\linewidth]{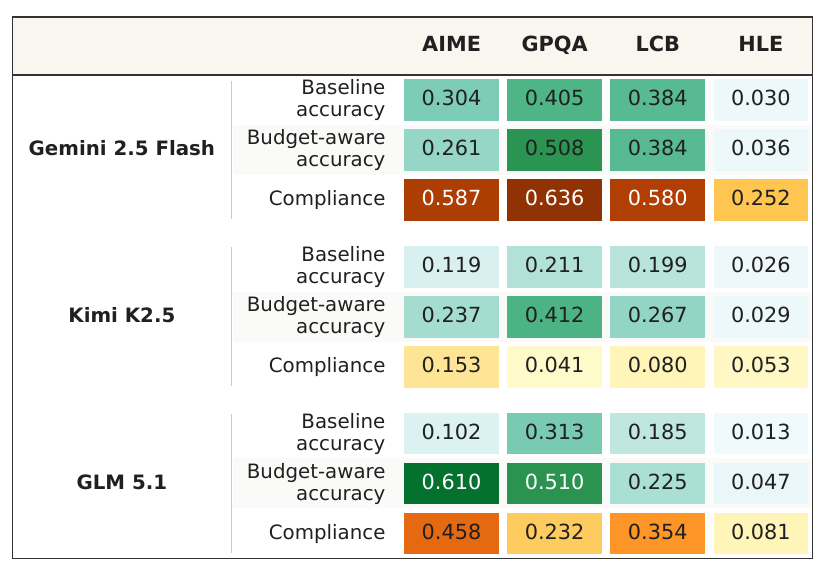}
  \caption{\textbf{Budget-aware re-solve, per (model $\times$ dataset).}
    HLE sub-domains are aggregated by problem-count weighted mean. Rows
    within each model block: reference accuracy, accuracy with
    the budget banner, and compliance rate. Color encoding follows
    Figure~\ref{fig:ba_overall}.}
  \label{fig:ba_per_dataset}
\end{figure}

On these selected-item subsets, accuracy increases from 14.3\% to 16.2\% for Gemini, from 8.4\% to 13.3\% for Kimi, and from 8.8\% to 17.1\% for GLM. Compliance with the stated allocation is 36.6\%, 6.0\%, and 16.0\%, respectively. On AIME, GLM's accuracy increases from $0.10$ to $0.61$. The low compliance shows that a verbal allocation instruction alone does not reliably enforce the requested resource commitment. Because re-solving samples a new completion and changes its prompt, the accuracy differences combine sampling and instruction effects. These subset accuracies are not full-benchmark pass@1 results and do not replace the fixed-reference U/E comparison.

\subsection{Exploratory feedback-based replanning}
\label{app:feedback-replanning}

An exploratory comparison evaluates Gemini~2.5 Flash, Kimi~K2.5, and GLM~5.1 in standard mode on two AIME pools each at $\alpha=1$. After each fresh solve, the feedback condition receives the outcome and remaining budget and can revise the plan. Replanning tokens are tracked separately and do not reduce the solving budget. The one-shot comparison uses the committed plan and reference outcomes, so the two conditions differ both in feedback and in whether solver completions are freshly sampled.

Replanning consumes 1.1--4.7 times as many tokens as solving, with mean 3.0. This ratio compares replanning input plus output tokens with solver output tokens. For Kimi~K2.5, replanning output alone amounts to about twice the entire solving budget $B$. Gemini's mean achieved value decreases from 7.5 to 6.0; Kimi and GLM improve. In three of the six pools, fresh solves exceed the fixed-reference oracle value, illustrating that this oracle bounds the recorded outcomes rather than all future samples. The small comparison therefore motivates further study of feedback and planning overhead, rather than a general claim that interactive replanning is better or worse than committed planning.

\subsection{Scaling and the full budget sweep}
\label{app:budget-sweep}

\paragraph{Scaling across the evaluated families.}
In Figure~\ref{fig:triage_efficiency}, advisory efficiency increases from 4B to 9B to 27B among the dense Qwen~3.5 models in standard mode on all four benchmarks. The 35B-A3B mixture-of-experts variant scores below the dense 27B model on each benchmark. Qwen~2.5 exhibits non-monotonic advisory efficiency across sizes. Larger parameter counts therefore do not uniformly yield better triage within the evaluated families.

\paragraph{Single-plan budget sweep.}
Tables~\ref{tab:sweep-aime-eta}--\ref{tab:sweep-hle-eta} report efficiency for the single-plan sweep at $\alpha\in\{0.25,0.5,0.75\}$. These available-pool summaries are distinct from the repeated-planning results in Figures~\ref{fig:triage_efficiency} and~\ref{fig:triage_regret}. At $\alpha=1$, the single-plan advisory and enforced regrets are 0.27 and 0.83; the corresponding repeated-planning means are 0.29 and 0.82.

\input{tables/sweep_aime_eta}
\clearpage
\input{tables/sweep_gpqa_diamond_eta}
\clearpage
\input{tables/sweep_livecodebench_eta}
\clearpage
\input{tables/sweep_hle_eta}

%% file: tables/allocation_counterfactual.tex
\begingroup
\small
\setlength{\tabcolsep}{3pt}
\renewcommand{\arraystretch}{1.08}
\begin{table}[htbp]
\centering
\caption{Allocation replacements on 1,143 half-budget plans at $\alpha=0.5$. Selection, order, and total budget remain fixed. The advisory row is the matched reference for gap recovery.}
\label{tab:allocation-counterfactual}
\begin{tabular}{lrr}
\toprule
Allocation rule & Efficiency & Gap recovered \\
\midrule
Declared allocations & $-0.57$ & \textnormal{--} \\
Reference cost + 20\% margin & $+0.12$ & 89\% \\
Equal split & $-0.13$ & 57\% \\
Advisory execution of the same plans & $+0.21$ & \textnormal{--} \\
\bottomrule
\end{tabular}
\end{table}
\endgroup

%% file: tables/decision_outcomes.tex
\begin{table}[htbp]
\centering
\small
\setlength{\tabcolsep}{3pt}
\caption{Pool-item outcomes in the single-plan half-budget evaluation. Percentages average within pools and then across configuration--benchmark comparisons; totals can differ from 100\% through rounding. Outcomes use fixed reference completions.}
\label{tab:decision-breakdown}
\begin{tabular}{lrr}
\toprule
Pool-item outcome & U & E \\
\midrule
Reached, reference cost covered, correct & 17\% & 5\% \\
Reached, reference cost covered, incorrect & 25\% & 5\% \\
Reached, allocation below reference cost & -- & 57\% \\
Selected but not reached before budget stop & 26\% & 2\% \\
Reference-correct but omitted & 8\% & 8\% \\
Reference-incorrect and omitted & 23\% & 23\% \\
\bottomrule
\end{tabular}
\end{table}

%% file: tables/reference_solved.tex
\begin{table}[htbp]
\centering
\small
\caption{Median triage efficiency at $\alpha=0.5$ for reference-solved control pools and normal pools from the same ten configurations. The control includes 71 valid plans, including three empty plans.}
\label{tab:reference-solved}
\begin{tabular}{lrr}
\toprule
Metric & Reference-solved pools & Normal pools \\
\midrule
Advisory efficiency $\eta_U$ & $+0.21$ & $+0.09$ \\
Enforced efficiency $\eta_E$ & $-1.69$ & $-0.92$ \\
\bottomrule
\end{tabular}
\end{table}

%% file: tables/unsolvable_matched_posted.tex
\begin{table}[htbp]
\centering
\small
\setlength{\tabcolsep}{5pt}
\renewcommand{\arraystretch}{0.9}
\caption{Expanded GPQA-Abstain evaluation: matched standard/extended comparisons. Detection $D_{\mathrm{listed}}$ is the fraction of injected items left out of the plan; $J_{\mathrm{listed}}$ also accounts for answerable-item coverage. Values are shown at two-decimal precision. The final row is the mean over the ten models. Higher is better.}
\label{tab:unsolvable-matched}
\begin{tabular}{lrrrr}
\toprule
& \multicolumn{2}{c}{$D_{\mathrm{listed}}$} & \multicolumn{2}{c}{$J_{\mathrm{listed}}$} \\
\cmidrule(lr){2-3}\cmidrule(lr){4-5}
Model & Std. & Ext. & Std. & Ext. \\
\midrule
Qwen 3.5 35B-A3B & 0.76 & 0.38 & 0.20 & 0.36 \\
Qwen 3.5 9B & 0.37 & 0.62 & 0.03 & 0.58 \\
Claude Haiku 4.5 & 0.33 & 0.65 & 0.34 & 0.62 \\
GLM 5.1 & 0.32 & 0.60 & 0.26 & 0.58 \\
Qwen 3.5 4B & 0.29 & 0.27 & 0.05 & 0.09 \\
Gemini 2.5 Flash & 0.25 & 0.70 & 0.30 & 0.56 \\
GPT-5 Mini & 0.18 & 0.21 & 0.12 & 0.28 \\
DeepSeek V3 & 0.09 & 0.56 & 0.09 & 0.59 \\
Kimi K2.5 & 0.00 & 0.52 & 0.00 & 0.61 \\
Qwen 3.5 27B & 0.00 & 0.20 & 0.00 & 0.27 \\
\midrule
Mean & 0.26 & 0.47 & 0.14 & 0.45 \\
\bottomrule
\end{tabular}
\end{table}

%% file: tables/prompt_framing.tex
\begingroup
\small
\setlength{\tabcolsep}{2pt}
\renewcommand{\arraystretch}{1.08}
\begin{table}[htbp]
\centering
\caption{Prompt-framing efficiencies and across-variant ranges for neutral (A), agent (B), and exam (C) variants at $\alpha\in\{0.25,0.5,0.75\}$. The evaluation covers eight dataset splits. Bold ranges are at least $0.10$.}
\label{tab:prompt-framing}
\begin{tabular}{lllrrrrrrrr}
\toprule
Model & Mode & $\alpha$ & \multicolumn{4}{c}{Advisory} & \multicolumn{4}{c}{Enforced} \\
\cmidrule(lr){4-7}\cmidrule(lr){8-11}
& & & $\eta^A_U$ & $\eta^B_U$ & $\eta^C_U$ & Range & $\eta^A_E$ & $\eta^B_E$ & $\eta^C_E$ & Range \\
\midrule
Kimi K2.5 & Std. & 0.25 & 0.615 & 0.613 & 0.575 & 0.040 & 0.152 & 0.152 & 0.148 & 0.004 \\
Kimi K2.5 & Std. & 0.50 & 0.530 & 0.659 & 0.554 & \textbf{0.130} & -0.254 & -0.232 & -0.245 & 0.022 \\
Kimi K2.5 & Std. & 0.75 & 0.681 & 0.781 & 0.776 & \textbf{0.100} & -1.421 & -1.352 & -1.421 & 0.069 \\
Kimi K2.5 & Ext. & 0.25 & 0.432 & 0.497 & 0.412 & 0.085 & -0.192 & -0.194 & -0.237 & 0.045 \\
Kimi K2.5 & Ext. & 0.50 & 0.463 & 0.477 & 0.526 & 0.063 & -0.712 & -0.692 & -0.700 & 0.020 \\
Kimi K2.5 & Ext. & 0.75 & 0.333 & 0.363 & 0.271 & 0.092 & -2.193 & -2.211 & -2.282 & 0.089 \\
Qwen 3.5 27B & Std. & 0.25 & 0.446 & 0.428 & 0.445 & 0.018 & 0.089 & 0.157 & 0.143 & 0.068 \\
Qwen 3.5 27B & Std. & 0.50 & 0.379 & 0.353 & 0.443 & 0.090 & -0.161 & -0.152 & -0.159 & 0.008 \\
Qwen 3.5 27B & Std. & 0.75 & 0.507 & 0.523 & 0.556 & 0.049 & -1.180 & -1.113 & -1.151 & 0.067 \\
Qwen 3.5 27B & Ext. & 0.25 & 0.414 & 0.361 & 0.391 & 0.053 & -0.110 & -0.111 & -0.105 & 0.005 \\
Qwen 3.5 27B & Ext. & 0.50 & 0.333 & 0.448 & 0.334 & \textbf{0.115} & -0.479 & -0.552 & -0.540 & 0.073 \\
Qwen 3.5 27B & Ext. & 0.75 & 0.287 & 0.415 & 0.475 & \textbf{0.188} & -1.568 & -1.597 & -1.782 & \textbf{0.214} \\
\bottomrule
\end{tabular}
\end{table}
\endgroup

%% file: tables/prompt_ranking_stability.tex
\begingroup
\small
\setlength{\tabcolsep}{4pt}
\begin{table}[htbp]
\centering
\caption{Kendall's $\tau_b$ between the rankings of the four model--mode configurations induced by each prompt-variant pair at $\alpha\in\{0.25,0.5,0.75\}$. Max range is the largest across-variant spread in any configuration at that $\alpha$.}
\label{tab:variant-tau}
\begin{tabular}{lrrrrrr}
\toprule
Regime & $\alpha$ & $\tau(A,B)$ & $\tau(A,C)$ & $\tau(B,C)$ & Min $\tau$ & Max range \\
\midrule
$\eta_U$ & 0.25 & 0.667 & 1.000 & 0.667 & 0.667 & 0.085 \\
$\eta_U$ & 0.50 & 0.667 & 1.000 & 0.667 & 0.667 & 0.130 \\
$\eta_U$ & 0.75 & 0.667 & 0.667 & 1.000 & 0.667 & 0.188 \\
$\eta_E$ & 0.25 & 0.667 & 1.000 & 0.667 & 0.667 & 0.068 \\
$\eta_E$ & 0.50 & 1.000 & 1.000 & 1.000 & 1.000 & 0.073 \\
$\eta_E$ & 0.75 & 1.000 & 1.000 & 1.000 & 1.000 & 0.214 \\
\bottomrule
\end{tabular}
\end{table}
\endgroup

%% file: tables/sweep_aime_eta.tex
\begingroup
\small
\setlength{\tabcolsep}{3pt}
\renewcommand{\arraystretch}{1.08}
\begin{table}[htbp]
\centering
\caption{Triage efficiency on AIME for the single-plan budget sweep. Entries average available scored pools. NA denotes an unavailable score due to an unparsable planning output.}
\label{tab:sweep-aime-eta}
\begin{tabular}{llrrrrrr}
\toprule
Model & Mode & $U_{0.25}$ & $U_{0.50}$ & $U_{0.75}$ & $E_{0.25}$ & $E_{0.50}$ & $E_{0.75}$ \\
\midrule
Qwen 2.5 3B & Std. & -0.33 & -0.95 & -2.65 & -0.33 & -0.95 & -2.65 \\
Qwen 2.5 7B & Std. & -0.09 & -0.31 & -2.11 & -0.31 & -0.63 & -2.74 \\
Qwen 2.5 14B & Std. & -0.05 & 0.41 & -1.26 & -0.32 & -0.75 & -1.63 \\
Qwen 2.5 32B & Std. & 0.27 & 0.20 & -0.40 & -0.18 & -0.55 & -2.35 \\
Qwen 3 4B & Ext. & 0.32 & 0.34 & 0.00 & -0.48 & -1.00 & -2.73 \\
Qwen 3 8B & Ext. & 0.30 & \textnormal{NA} & 0.25 & -0.46 & \textnormal{NA} & -2.75 \\
Qwen 3 14B & Ext. & 0.67 & 0.39 & 0.30 & -0.63 & -0.95 & -2.75 \\
Qwen 3 32B & Ext. & 0.50 & 0.36 & -0.26 & -0.56 & -1.06 & -2.79 \\
Qwen 3.5 4B & Std. & 0.01 & -0.47 & -1.49 & -0.32 & -0.96 & -2.78 \\
Qwen 3.5 4B & Ext. & 0.33 & -0.95 & 0.53 & -0.34 & -0.95 & -2.78 \\
Qwen 3.5 9B & Std. & -0.01 & -0.04 & -0.51 & -0.41 & -0.83 & -2.76 \\
Qwen 3.5 9B & Ext. & -0.36 & -0.22 & 0.06 & -0.36 & -0.95 & -1.51 \\
Qwen 3.5 27B & Std. & 0.36 & 0.43 & 0.52 & -0.21 & -0.22 & -0.64 \\
Qwen 3.5 27B & Ext. & 0.53 & -0.05 & -1.56 & -0.64 & -0.88 & -2.09 \\
Qwen 3.5 35B-A3B & Std. & -0.32 & -0.51 & -1.32 & -0.38 & -0.96 & -2.76 \\
Qwen 3.5 35B-A3B & Ext. & -0.65 & 0.02 & 0.06 & -0.65 & -0.83 & -1.59 \\
Claude Haiku 4.5 & Std. & -0.12 & -0.20 & 0.48 & -0.59 & -0.73 & -1.17 \\
Claude Haiku 4.5 & Ext. & 0.80 & 0.41 & -1.09 & -0.68 & -0.83 & -2.75 \\
DeepSeek V3 & Std. & 0.46 & 0.21 & -0.22 & -0.22 & -0.39 & -0.97 \\
DeepSeek V3 & Ext. & 0.68 & 0.69 & -1.07 & -0.56 & -1.11 & -2.73 \\
Gemini 2.5 Flash & Std. & 0.37 & 0.54 & 0.52 & -0.02 & 0.07 & -1.06 \\
Gemini 2.5 Flash & Ext. & 0.60 & 0.03 & -2.24 & -1.14 & -2.04 & -3.30 \\
GPT-5 Mini & Std. & 0.03 & 0.12 & -0.21 & -0.39 & -0.72 & -1.28 \\
GPT-5 Mini & Ext. & -0.15 & -0.02 & -0.81 & -0.57 & -0.53 & -2.87 \\
Kimi K2.5 & Std. & 0.74 & 0.23 & 0.63 & -0.31 & -0.96 & -2.75 \\
Kimi K2.5 & Ext. & 0.67 & 0.58 & 0.63 & -0.43 & -0.88 & -2.58 \\
GPT-OSS 20B & Ext. & -0.20 & 0.06 & -0.33 & -0.46 & -0.69 & -2.30 \\
GPT-OSS 120B & Ext. & -0.15 & 0.19 & 0.06 & -0.59 & -0.36 & -1.35 \\
GLM 5.1 & Std. & 0.51 & 0.76 & 0.56 & -0.30 & -0.71 & -1.78 \\
GLM 5.1 & Ext. & 0.48 & 0.57 & 0.26 & -0.44 & -1.01 & -2.73 \\
\bottomrule
\end{tabular}
\end{table}
\endgroup

%% file: tables/sweep_gpqa_diamond_eta.tex
\begingroup
\small
\setlength{\tabcolsep}{3pt}
\renewcommand{\arraystretch}{1.08}
\begin{table}[htbp]
\centering
\caption{Triage efficiency on GPQA Diamond for the single-plan budget sweep. Entries average available scored pools.}
\label{tab:sweep-gpqa_diamond-eta}
\begin{tabular}{llrrrrrr}
\toprule
Model & Mode & $U_{0.25}$ & $U_{0.50}$ & $U_{0.75}$ & $E_{0.25}$ & $E_{0.50}$ & $E_{0.75}$ \\
\midrule
Qwen 2.5 3B & Std. & -0.07 & -0.18 & -0.67 & -0.07 & -0.26 & -1.20 \\
Qwen 2.5 7B & Std. & -0.14 & -0.08 & -0.64 & 0.00 & -0.03 & -1.04 \\
Qwen 2.5 14B & Std. & 0.01 & -0.10 & -0.56 & -0.14 & -0.55 & -1.79 \\
Qwen 2.5 32B & Std. & -0.15 & -0.48 & -1.13 & -0.49 & -1.02 & -2.38 \\
Qwen 3 4B & Ext. & 0.11 & 0.18 & 0.29 & -0.39 & -0.97 & -2.72 \\
Qwen 3 8B & Ext. & 0.09 & 0.06 & 0.04 & -0.31 & -0.84 & -1.67 \\
Qwen 3 14B & Ext. & 0.08 & 0.07 & -0.06 & -0.45 & -0.97 & -2.64 \\
Qwen 3 32B & Ext. & 0.14 & 0.15 & -0.35 & -0.47 & -0.84 & -2.60 \\
Qwen 3.5 4B & Std. & 0.05 & -0.16 & -0.13 & -0.32 & -0.80 & -2.40 \\
Qwen 3.5 4B & Ext. & -0.15 & -0.09 & -0.27 & -0.59 & -1.14 & -2.71 \\
Qwen 3.5 9B & Std. & 0.14 & -0.09 & 0.16 & -0.41 & -0.81 & -2.29 \\
Qwen 3.5 9B & Ext. & 0.11 & -0.07 & -0.08 & -0.68 & -1.21 & -2.47 \\
Qwen 3.5 27B & Std. & 0.25 & 0.09 & -0.25 & 0.21 & 0.01 & -0.83 \\
Qwen 3.5 27B & Ext. & 0.23 & 0.20 & 0.25 & -0.75 & -1.18 & -1.92 \\
Qwen 3.5 35B-A3B & Std. & -0.54 & -0.94 & -1.98 & -0.63 & -0.94 & -1.98 \\
Qwen 3.5 35B-A3B & Ext. & 0.26 & -0.15 & 0.28 & -0.73 & -1.21 & -2.29 \\
Claude Haiku 4.5 & Std. & 0.44 & 0.21 & -0.33 & -0.66 & -0.44 & -0.89 \\
Claude Haiku 4.5 & Ext. & 0.55 & 0.17 & -0.15 & -0.71 & -0.94 & -2.42 \\
DeepSeek V3 & Std. & 0.25 & 0.20 & 0.06 & -0.42 & -0.63 & -1.47 \\
DeepSeek V3 & Ext. & 0.53 & 0.44 & 0.21 & -0.51 & -1.06 & -2.54 \\
Gemini 2.5 Flash & Std. & 0.18 & 0.10 & -0.40 & 0.23 & 0.07 & -0.69 \\
Gemini 2.5 Flash & Ext. & 0.45 & 0.29 & -0.51 & -0.81 & -1.37 & -2.92 \\
GPT-5 Mini & Std. & 0.24 & -0.19 & -0.09 & -0.25 & -0.12 & -0.80 \\
GPT-5 Mini & Ext. & 0.32 & 0.21 & 0.18 & -0.17 & -0.56 & -2.31 \\
Kimi K2.5 & Std. & 0.33 & 0.22 & 0.44 & -0.28 & -0.90 & -2.71 \\
Kimi K2.5 & Ext. & 0.39 & 0.40 & 0.01 & -0.48 & -0.97 & -2.68 \\
GPT-OSS 20B & Ext. & 0.03 & -0.07 & -0.05 & -0.33 & -0.90 & -2.43 \\
GPT-OSS 120B & Ext. & 0.02 & -0.11 & -0.09 & -0.26 & -0.69 & -1.57 \\
GLM 5.1 & Std. & 0.49 & 0.68 & 0.52 & -0.32 & -0.82 & -2.17 \\
GLM 5.1 & Ext. & 0.54 & 0.72 & 0.55 & -0.41 & -1.02 & -2.53 \\
\bottomrule
\end{tabular}
\end{table}
\endgroup

%% file: tables/sweep_livecodebench_eta.tex
\begingroup
\small
\setlength{\tabcolsep}{3pt}
\renewcommand{\arraystretch}{1.08}
\begin{table}[htbp]
\centering
\caption{Triage efficiency on LiveCodeBench for the single-plan budget sweep. Entries average available scored pools.}
\label{tab:sweep-livecodebench-eta}
\begin{tabular}{llrrrrrr}
\toprule
Model & Mode & $U_{0.25}$ & $U_{0.50}$ & $U_{0.75}$ & $E_{0.25}$ & $E_{0.50}$ & $E_{0.75}$ \\
\midrule
Qwen 2.5 3B & Std. & 0.01 & -0.26 & -1.51 & 0.12 & -0.49 & -1.51 \\
Qwen 2.5 7B & Std. & -0.01 & -0.28 & -0.28 & -0.06 & -0.44 & -0.66 \\
Qwen 2.5 14B & Std. & 0.03 & -0.22 & -1.63 & -0.04 & -0.35 & -2.16 \\
Qwen 2.5 32B & Std. & 0.29 & -0.24 & -1.13 & -0.03 & -0.51 & -1.20 \\
Qwen 3 4B & Ext. & -0.18 & -0.91 & -2.46 & -0.33 & -0.96 & -2.68 \\
Qwen 3 8B & Ext. & -0.27 & -0.94 & -2.42 & -0.33 & -0.94 & -2.67 \\
Qwen 3 14B & Ext. & -0.25 & -0.21 & -2.03 & -0.34 & -0.95 & -2.77 \\
Qwen 3 32B & Ext. & 0.25 & -0.13 & -0.58 & -0.32 & -0.94 & -2.75 \\
Qwen 3.5 4B & Std. & -0.34 & -0.96 & -2.71 & -0.34 & -0.96 & -2.71 \\
Qwen 3.5 4B & Ext. & -0.14 & -0.95 & -2.83 & -0.31 & -0.95 & -2.83 \\
Qwen 3.5 9B & Std. & 0.06 & -0.00 & -0.39 & 0.10 & -0.51 & -2.00 \\
Qwen 3.5 9B & Ext. & -0.21 & -0.63 & -0.67 & -0.32 & -0.80 & -2.85 \\
Qwen 3.5 27B & Std. & 0.41 & 0.46 & 0.37 & 0.15 & 0.15 & -0.48 \\
Qwen 3.5 27B & Ext. & 0.19 & -0.26 & -0.90 & -0.39 & -0.95 & -2.58 \\
Qwen 3.5 35B-A3B & Std. & 0.03 & -0.28 & -0.36 & -0.03 & -0.47 & -1.21 \\
Qwen 3.5 35B-A3B & Ext. & 0.43 & -0.11 & 0.02 & -0.36 & -0.95 & -2.30 \\
Claude Haiku 4.5 & Std. & 0.45 & 0.33 & 0.22 & -0.35 & -0.62 & -1.46 \\
Claude Haiku 4.5 & Ext. & 0.42 & 0.23 & -0.24 & -0.54 & -0.92 & -2.51 \\
DeepSeek V3 & Std. & 0.21 & 0.37 & 0.32 & -0.12 & -0.18 & -0.74 \\
DeepSeek V3 & Ext. & 0.12 & -0.69 & -1.19 & -0.32 & -0.94 & -2.76 \\
Gemini 2.5 Flash & Std. & 0.51 & 0.49 & 0.47 & 0.39 & 0.36 & 0.05 \\
Gemini 2.5 Flash & Ext. & 0.50 & 0.29 & -0.42 & -0.14 & -0.42 & -1.56 \\
GPT-5 Mini & Std. & -0.10 & -0.15 & -2.23 & -0.21 & -0.31 & -2.50 \\
GPT-5 Mini & Ext. & -0.20 & -0.44 & -1.58 & -0.26 & -0.64 & -2.21 \\
Kimi K2.5 & Std. & 0.81 & 0.77 & 0.92 & -0.34 & -0.94 & -2.79 \\
Kimi K2.5 & Ext. & 0.67 & 0.64 & 0.37 & -0.33 & -0.97 & -2.54 \\
GPT-OSS 20B & Ext. & -0.21 & -0.23 & -0.86 & -0.30 & -0.82 & -2.07 \\
GPT-OSS 120B & Ext. & -0.25 & -0.80 & -2.64 & -0.29 & -0.93 & -2.64 \\
GLM 5.1 & Std. & 0.85 & 0.88 & 0.84 & -0.26 & -0.62 & -1.93 \\
GLM 5.1 & Ext. & 0.93 & 0.93 & 1.00 & 0.01 & -0.29 & -1.73 \\
\bottomrule
\end{tabular}
\end{table}
\endgroup

%% file: tables/sweep_hle_eta.tex
\begingroup
\small
\setlength{\tabcolsep}{3pt}
\renewcommand{\arraystretch}{1.08}
\begin{table}[htbp]
\centering
\caption{Triage efficiency on Humanity's Last Exam for the single-plan budget sweep. Entries average available scored pools.}
\label{tab:sweep-hle-eta}
\begin{tabular}{llrrrrrr}
\toprule
Model & Mode & $U_{0.25}$ & $U_{0.50}$ & $U_{0.75}$ & $E_{0.25}$ & $E_{0.50}$ & $E_{0.75}$ \\
\midrule
Qwen 2.5 3B & Std. & 0.36 & 0.16 & -0.25 & 0.19 & -0.03 & -0.82 \\
Qwen 2.5 7B & Std. & 0.53 & 0.32 & -0.27 & 0.44 & 0.21 & -0.61 \\
Qwen 2.5 14B & Std. & 0.23 & -0.11 & -0.25 & 0.08 & -0.27 & -1.24 \\
Qwen 2.5 32B & Std. & 0.49 & 0.35 & -0.21 & 0.18 & -0.16 & -1.08 \\
Qwen 3 4B & Ext. & 0.33 & 0.23 & -0.10 & 0.03 & -0.13 & -1.14 \\
Qwen 3 8B & Ext. & 0.24 & 0.25 & 0.69 & -0.01 & -0.21 & -0.79 \\
Qwen 3 14B & Ext. & 0.29 & 0.36 & 0.05 & 0.12 & -0.34 & -1.22 \\
Qwen 3 32B & Ext. & 0.38 & 0.24 & 0.17 & 0.22 & -0.37 & -1.17 \\
Qwen 3.5 4B & Std. & 0.27 & 0.06 & -0.13 & 0.09 & -0.35 & -1.53 \\
Qwen 3.5 4B & Ext. & 0.22 & 0.00 & 0.02 & 0.07 & -0.10 & -0.78 \\
Qwen 3.5 9B & Std. & 0.48 & 0.36 & 0.44 & 0.30 & -0.02 & -0.94 \\
Qwen 3.5 9B & Ext. & 0.55 & 0.36 & 0.53 & 0.29 & -0.22 & -1.13 \\
Qwen 3.5 27B & Std. & 0.43 & 0.41 & 0.66 & 0.07 & -0.25 & -1.37 \\
Qwen 3.5 27B & Ext. & 0.51 & 0.58 & 0.91 & 0.08 & -0.36 & -1.15 \\
Qwen 3.5 35B-A3B & Std. & 0.04 & -0.36 & -1.74 & -0.01 & -0.46 & -1.74 \\
Qwen 3.5 35B-A3B & Ext. & 0.24 & 0.71 & 0.66 & -0.07 & -0.29 & -1.58 \\
Claude Haiku 4.5 & Std. & 0.53 & 0.37 & 0.32 & 0.32 & 0.23 & 0.10 \\
Claude Haiku 4.5 & Ext. & 0.63 & 0.77 & 0.25 & 0.41 & 0.16 & -0.55 \\
DeepSeek V3 & Std. & 0.32 & 0.55 & 0.03 & 0.27 & -0.10 & -0.76 \\
DeepSeek V3 & Ext. & 0.48 & 0.49 & -0.05 & -0.11 & -0.61 & -2.10 \\
Gemini 2.5 Flash & Std. & 0.65 & 0.70 & 0.76 & 0.42 & 0.28 & -0.32 \\
Gemini 2.5 Flash & Ext. & 0.32 & 0.50 & 0.64 & 0.03 & -0.38 & -1.42 \\
GPT-5 Mini & Std. & 0.30 & 0.26 & 0.10 & -0.00 & -0.45 & -1.58 \\
GPT-5 Mini & Ext. & 0.27 & 0.28 & -0.31 & -0.09 & -0.62 & -2.07 \\
Kimi K2.5 & Std. & 0.60 & 0.54 & 0.69 & 0.38 & 0.09 & -0.74 \\
Kimi K2.5 & Ext. & 0.35 & 0.37 & 0.39 & -0.08 & -0.64 & -2.04 \\
GPT-OSS 20B & Ext. & 0.43 & 0.30 & 0.31 & 0.28 & -0.07 & -1.00 \\
GPT-OSS 120B & Ext. & 0.41 & 0.47 & -0.11 & 0.30 & 0.27 & -0.81 \\
GLM 5.1 & Std. & 0.82 & 0.71 & 0.87 & 0.60 & 0.42 & -0.11 \\
GLM 5.1 & Ext. & 0.39 & 0.32 & 0.31 & -0.17 & -0.78 & -2.20 \\
\bottomrule
\end{tabular}
\end{table}
\endgroup